\documentclass[english]{article}
\usepackage{lmodern}
\renewcommand{\sfdefault}{lmss}

\usepackage[T1]{fontenc}
\usepackage[a4paper]{geometry}
\usepackage{dsfont}
\usepackage{color}
\usepackage{babel}
\usepackage{amsmath}
\usepackage{setspace}
\usepackage{amsthm}
\usepackage{makecell}
\usepackage{amssymb}
\usepackage{stmaryrd}
\usepackage{graphicx}
\usepackage{esint}
\usepackage{caption}

\usepackage[authoryear]{natbib}
\usepackage[unicode=true,bookmarks=true,bookmarksnumbered=false,bookmarksopen=false,breaklinks=false,pdfborder={0 0 0},pdfborderstyle={},backref=page,colorlinks=true]{hyperref}
\hypersetup{pdftitle={Diffusion-PINN Sampler},pdfauthor={Zhekun Shi, Longlin Yu, Tianyu Xie, Cheng Zhang},linkcolor=RoyalBlue,citecolor=RoyalBlue}
\usepackage{float}
\usepackage{subfig}
\usepackage[dvipsnames,svgnames,x11names,hyperref]{xcolor}
\usepackage{nicefrac}

\usepackage{epstopdf}
\usepackage{url} 
\usepackage{caption}
\usepackage{multirow}
\usepackage{amssymb}
\usepackage{multicol}
\usepackage{booktabs}
\usepackage{tikz}
\usepackage{bbm}
\usepackage{bm}
\usepackage{algorithm}
\usepackage{algorithmic}

\newtheorem{lemma}{Lemma}
\newtheorem{theorem}{Theorem}
\newtheorem{example}{Example}

\newtheorem{proposition}{Proposition}
\newtheorem{remark}{Remark}

\newtheorem{assumption}{Assumption}

\usepackage[textsize=tiny]{todonotes}

\makeatletter
\newcommand{\myfnsymbol}[1]{%
  \expandafter\@myfnsymbol\csname c@#1\endcsname
}
\newcommand{\@myfnsymbol}[1]{%
  \ifcase #1
  \or 1
  \or 2
  \or 3
  \or 4
  \or \TextOrMath{\textasteriskcentered}{*}
  \or \TextOrMath{\textdagger}{\dagger}
  \fi
}
\newcommand{\affiliationA}{\@myfnsymbol{1}}
\newcommand{\affiliationB}{\@myfnsymbol{2}}
\newcommand{\affiliationC}{\@myfnsymbol{3}}
\newcommand{\affiliationD}{\@myfnsymbol{4}}
\newcommand{\correspondingA}{\@myfnsymbol{6}}
\makeatother

\usepackage{amsmath,amsfonts,bm}

\def\1{\bm{1}}

\def\eps{{\varepsilon}}

\def\rd{{\textnormal{d}}}

\DeclareMathAlphabet{\mathsfit}{\encodingdefault}{\sfdefault}{m}{sl}
\SetMathAlphabet{\mathsfit}{bold}{\encodingdefault}{\sfdefault}{bx}{n}

\usepackage{url}

\usepackage[utf8]{inputenc} 
\usepackage[T1]{fontenc}    
\usepackage{booktabs}       
\usepackage{amsfonts}       
\usepackage{nicefrac}       
\usepackage{microtype}      
\usepackage{xcolor}         
\usepackage{diagbox}
\usepackage{graphicx}
\usepackage{natbib}
\usepackage{blindtext}
\usepackage{amsmath}
\usepackage{amsthm}
\usepackage{amssymb}
\usepackage[amsstyle]{cases}
\usepackage{titlesec}
\usepackage{bm}
\usepackage{bbm}
\usepackage{latexsym}
\usepackage{multirow}
\usepackage{enumitem}
\usepackage{algorithm}
\usepackage{algorithmic}
\usepackage{wrapfig}
\usepackage{float}
\usepackage{bbding}
\usepackage[most]{tcolorbox}

\newcommand{\bmx}{\bm{x}}
\newcommand{\bmz}{\bm{z}}
\newcommand{\bR}{\mathbb{R}}
\newcommand{\bmB}{\bm{B}}
\newcommand{\bmf}{\bm{f}}
\newcommand{\bms}{\bm{s}}
\newcommand{\bma}{\bm{a}}

\newcommand{\bE}{\mathbb{E}}
\newcommand{\cN}{\mathcal{N}}
\newcommand{\rmd}{\mathrm{~d}}

\newcommand{\nx}{\nabla_{\bm{x}}}
\newcommand{\pt}{\partial_t}
\newcommand{\Lap}{\Delta}
\newcommand{\diver}{\nabla\cdot}
\newcommand{\cC}{\mathcal{C}}
\newcommand{\cL}{\mathcal{L}}
\newcommand{\hB}{\widehat{B}}
\newcommand{\Bv}{B^{\nu}}
\newcommand{\Bu}{B^{*}}
\newcommand{\vareps}{\varepsilon}
\newcommand{\bbmone}{\mathbbm{1}}
\newcommand{\les}{\leqslant}
\newcommand{\ges}{\geqslant}

\begin{document}

\title{Diffusion-PINN Sampler}

\author{
Zhekun Shi\textsuperscript{\affiliationA},
Longlin Yu\textsuperscript{\affiliationB},
Tianyu Xie\textsuperscript{\affiliationC},
Cheng Zhang\textsuperscript{\affiliationD,\correspondingA}
}

\date{
}

\renewcommand{\thefootnote}{\myfnsymbol{footnote}}
\maketitle
\footnotetext[1]{School of Mathematical Sciences, Peking University, Beijing, 100871, China. Email: zhekunshi@stu.pku.edu.cn}%
\footnotetext[2]{School of Mathematical Sciences, Peking University, Beijing, 100871, China. Email: llyu@pku.edu.cn}%
\footnotetext[3]{School of Mathematical Sciences, Peking University, Beijing, 100871, China. Email: tianyuxie@pku.edu.cn}%
\footnotetext[4]{School of Mathematical Sciences and Center for Statistical Science, Peking University, Beijing, 100871, China. Email: chengzhang@math.pku.edu.cn}
\footnotetext[6]{Corresponding author}%

\setcounter{footnote}{0}
\renewcommand{\thefootnote}{\fnsymbol{footnote}}

\begin{abstract}
Recent success of diffusion models has inspired a surge of interest in developing sampling techniques using reverse diffusion processes.
However, accurately estimating the drift term in the reverse stochastic differential equation (SDE) solely from the unnormalized target density poses significant challenges, hindering existing methods from achieving state-of-the-art performance.
In this paper, we introduce the \emph{Diffusion-PINN Sampler} (DPS), a novel diffusion-based sampling algorithm that estimates the drift term by solving the governing partial differential equation of the log-density of the underlying SDE marginals via physics-informed neural networks (PINN).
We prove that the error of log-density approximation can be controlled by the PINN residual loss, enabling us to establish convergence guarantees of DPS.
Experiments on a variety of sampling tasks demonstrate the effectiveness of our approach, particularly in accurately identifying mixing proportions when the target contains isolated components.
  
\end{abstract}


\maketitle
\begin{spacing}{1}
    \tableofcontents
\end{spacing}

\section{Introduction}

Sampling from unnormalized distributions is a fundamental yet challenging task encountered across various scientific disciplines such as Bayesian statistics, computational physics, chemistry, and biology \citep{liu2001,stoltz2010}.
Markov chain Monte Carlo (MCMC) and variational inference (VI) have historically been the go-to methods for this problem. However, these approaches exhibit limitations when dealing with complex target distributions (e.g., distributions with multimodality or heavy tails).
Recently, the success of diffusion models for generative modeling \citep{song2020score, ho2020denoising, pmlr-v139-nichol21a, VDM} have sparked considerable interest in tackling the sampling problem using the reverse diffusion processes that transport a given prior density to the target, governed by stochastic differential equations (SDE).

In diffusion-based generative models, the score function in the drift term of the reverse SDE is learned based on score matching techniques \citep{scorematching, vincent2011connection} that require samples from the target data distribution.
However, for sampling tasks, we only have access to an unnormalized density function $\pi$, making it challenging to estimate the score function for the reverse SDE.
From a stochastic optimal control perspective \citep{pmlr-v99-tzen19a,NEURIPS2021_940392f5}, several VI methods that parameterize the drift term with neural network approximation have been proposed \citep{zhang2021path, berner2022optimal, vargas2023,vargas2023denoising}.
Nevertheless, these approaches face challenges such as instability during training, the computational complexity associated with differentiating through SDE solvers, and mode collapse issues arising from training objectives based on reverse Kullback-Leibler (KL) divergences \citep{zhang2021path, vargas2023denoising}.
On the other hand, \citet{huang2023reverse} proposed a scheme based on the connection between score matching and non-parametric posterior mean estimation.
More specifically, they use MCMC estimation of the scores to potentially alleviate the numerical bias intrinsic in parametric estimation methods such as neural networks.
However, this method also introduces noise in the estimates and requires repetitive posterior sampling in each time step of the reverse SDE.
Overall, despite their potential, diffusion-based sampling methods have not yet achieved state-of-the-art performance.

In addition to its connection with posterior mean estimation, the score function has also been shown to evolve according to a partial differential equation known as the \emph{score Fokker Planck equation} (score FPE) \citep{lai2023fp}.
This discovery has led to a novel regularization technique for enhancing score function estimation in diffusion models \citep{lai2023fp,deveney2023closing}.
In this paper, we adopt this strategy for diffusion-based sampling methods.
While the score function can be recovered by solving the score FPE using the score of target distribution $\pi$ as the initial condition, we demonstrate that it may fail to identify correct mixing proportions when $\pi$ has isolated components, a common limitation of score-based methods \citep{wenliang2020, zhang2022towards}.
To remedy this issue, we propose to solve the log-density FPE, a similar partial differential equation for the log-density function, using physics-informed neural networks (PINN) \citep{raissi2019physics, wang20222}.
The estimated log-density function is then integrated into the reverse SDE, leading to a novel sampling algorithm termed \emph{Diffusion-PINN Sampler} (DPS).
We prove that the error of log-density estimation can be controlled by the PINN residual loss, which allows us to ensure convergence guarantee of DPS based on established results for score-based generative models \citep{chen2023sampling, chen2023improved,benton2023linear}.
Experiments on a variety of sampling tasks provide compelling numerical evidence for the superiority of our method compared to other baseline methods.

\section{Related Works}
Recently, several works have explored the combination of Physics-Informed Neural Networks (PINN) and sampling techniques. 
For instance, \citet{mate2023learning, fanpath, tian2024liouville} address the continuity equation using PINN based on ODEs and achieve flow-based sampling through a linear interpolation (i.e., annealing) path between the target distribution and a simple prior, such as a Gaussian distribution.
Besides, \citet{berner2022optimal} (in the appendix of their paper) and \citet{sun2024physics} propose solving the log-density Hamilton–Jacobi–Bellman (HJB) equation via PINN to develop a SDE-based sampling algorithm. 
However, both approaches lack comprehensive numerical investigation and thorough theoretical analysis.
In contrast, our work investigates a limitation of score-based Fokker-Planck equations (FPE) in identifying the mixing proportions of multi-modal distributions, introduces novel computational techniques for solving PDEs via PINN in the context of diffusion-based sampling, and provides the first complete theoretical analysis of the algorithm.

\section{Background}
\paragraph{Notations.}
Throughout the paper, $\Omega\subset\bR^d$ denotes a bounded and closed domain. 
For simplicity, we do not distinguish a probabilistic measure from its density function.
We use $\bmx=(x_1,\cdots, x_d)^{\prime}$ to denote a vector in $\bR^d$ and $\|\bmx\|=\sqrt{x_1^2+\cdots x_d^2}$ stands for the $L^2$-norm.
Let $\nu$ denote a probability measure on $\bR^d$, for any $\bmf:\bR^d\to\bR^m$, we denote $\|\bmf(\cdot)\|^2_{L^2(\Omega;\nu)}:=\int_{\Omega}\|\bmf(\bmx)\|^2\rmd\nu(\bmx)$.
For any $\bmf:\bR^d\times [0, T]\to\bR^m$ , we define $\|\bmf_t(\cdot)\|_{L^2(\Omega;\nu)}^2:=\int_{\Omega}\|\bmf_t(\bmx)\|^2\rmd\nu(\bmx)$ as a function of $t\in [0, T]$.
For any $\bm{F}=(F_1,\cdots, F_d)^{\prime}:\bR^d\to\bR^d$, we denote the divergence of $\bm{F}$ by $\diver\bm{F}:=\sum_{i=1}^d\partial_{x_i}F_i$.
For any $F:\bR^d\to\bR$, we denote the Laplacian of $F$ by $\Lap F:=\sum_{i=1}^d\partial^2_{x_i}F$.

\paragraph{Diffusion models.}
In diffusion models, noise is progressively added to the training samples via a forward stochastic process described by the following stochastic differential equation (SDE)
\begin{equation}\label{eq: diffusion forward}
    \rmd\bmx_t = \bmf(\bmx_t,t)\rmd t + g(t)\rmd\bmB_t,\quad \bmx_0\sim p_0(\cdot),\quad 0\les t\les T,
\end{equation}
where $p_0(\cdot)$ is the data distribution, $\bmB_t$ is a standard Brownian motion, and $\bmf(\bmx_t,t)$ and $g(t)$ are the drift and diffusion coefficients respectively. 
The derivatives of the log-density of the forward marginals, i.e., \emph{scores}, are learned via score matching techniques \citep{vincent2011connection, song2020score} and new samples from the data distribution can be obtained by simulating the following reverse process
\begin{equation}\label{eq: diffusion backward}
\rmd\bmx_t = \left[\bmf(\bmx_t,t) - g^2(t)\nabla_{\bmx_t}\log p_t(\bmx_t)\right]\rmd t + g(t)\rmd\Bar{\bmB}_t,\quad \bmx_T\sim p_T(\cdot),
\end{equation}
where $p_t(\cdot)$ is the probability density of $\bmx_t$ and $\Bar{\bmB}_t$ is a standard Brownian motion from $T$ to $0$.

\paragraph{Physics-informed neural networks (PINN).}
PINN is a deep learning method for solving partial differential equations (PDEs) \citep{raissi2019physics}.
Consider the following general form of PDE
\begin{subequations}
\arraycolsep=1.8pt
\def\arraystretch{1.5}
    \label{eq: pinn pdes}
    \begin{align}
        \mathcal{L}u(\bmx) = \varphi(\bmx), \quad & \bmx\in\Omega\subseteq\mathbb{R}^d,\\
        \mathcal{B}u(\bmx) = \psi(\bmx),\quad & \bmx\in\partial\Omega,
    \end{align}
\end{subequations}
where $\mathcal{L}$ and $\mathcal{B}$ are the differential operators on domain $\Omega$ and boundary $\partial\Omega$, respectively.
PINN seeks an approximate solution using deep model $u_\theta(\bmx)$ by minimizing the $L^2$ PINN residual losses
\begin{subequations}
\arraycolsep=1.8pt
\def\arraystretch{1.5}
\label{eq: pinn loss}
\begin{align}
    \ell_{\Omega}(u_\theta)&:=\left\|\mathcal{L}u_\theta(\bmx)-\varphi(\bmx)\right\|^2_{L^2(\Omega;\nu)},\label{eq: pinn loss a}\\
    \ell_{\partial\Omega}(u_\theta)&:=\left\|\mathcal{B}u_\theta(\bmx) - \psi(\bmx)\right\|^2_{L^2(\Omega;\nu)}, \label{eq: pinn loss b}
\end{align}
\end{subequations}
where $\nu$ is a probability measure for collocation points generation, often taken to be the uniform distribution on $\Omega$.
The two terms $\ell_{\Omega}(u)$ and $\ell_{\partial\Omega}(u)$ in Eq. \eqref{eq: pinn loss} reflect the approximation error on $\Omega$ and $\partial\Omega$ respectively.
In practice, the losses in Eq. \eqref{eq: pinn loss} can be optimized by gradient-based methods with Monte Carlo gradient estimation.

\paragraph{Fokker Planck equation.}
The evolution of the density $p_t(\bmx)$ associated with the forward SDE \eqref{eq: diffusion forward} is governed by the Fokker Planck equation (FPE) \citep{oksendal2003stochastic}
\begin{equation}\label{eq: fokker planck equaiton}
\pt p_t(\bmx)= \underbrace{\frac{1}{2}g^2(t)\Lap p_t(\bmx) - \diver\left[\bmf(\bmx,t)p_t(\bmx)\right]}_{:= \mathcal{L}_{\textrm{FPE}}p_t(\bmx)}.
\end{equation}
Recently, \citet{lai2023fp} derive an equivalent system of PDEs for the log density $\log p_t(\bmx)$ and score $\nx\log p_t(\bmx)$, termed as the log-density Fokker Planck equation (log-density FPE) and the score Fokker Planck equation (score FPE) respectively, as summarized in Theorem \ref{thm: fp} (the proof can be found in Appendix \ref{subsection: proof of log-density/score fpe}).
\begin{theorem}[Log-density FPE and score FPE; Proposition 3.1 in \citet{lai2023fp}]\label{thm: fp}
 Assume the density $p_t(\bmx)$ is sufficiently smooth on $\mathbb{R}^d\times[0,T]$. Then for all $(\bmx, t) \in \mathbb{R}^d\times[0,T]$, the log-density $u_t(\bmx):=\log p_t(\bmx)$ satisfies the PDE
\begin{equation}\label{eq:log-fp}
\partial_t u_t(\bmx) = \underbrace{
     \frac{1}{2}g^2(t) \Lap u_t(\bmx) 
    +\frac{1}{2}g^2(t)\left\|\nx u_t(\bmx)\right\|^2
    -\bmf(\bmx, t)\cdot \nx u_t(\bmx) 
    - \diver\bmf(\bmx, t)}_{:=\mathcal{L}_{\emph{\textrm{L-FPE}}}u_t(\bmx)},
\end{equation}
and the score $\bms_t(\bmx):=\nabla_{\bmx}\log p_t(\bmx)$ satisfies the PDE
\begin{equation}\label{eq:score-fp}
\partial_t \bms_t(\bmx) = \underbrace{\nabla_{\bmx}\left[
\frac{1}{2}g^2(t) \diver\bms_t(\bmx) 
+\frac{1}{2}g^2(t)\left\|\bms_t(\bmx)\right\|^2
-\bmf(\bmx, t)\cdot \bms_t(\bmx)
- \nabla\cdot \bmf(\bmx, t) \right]}_{:=\mathcal{L}_{\emph{\textrm{S-FPE}}}\bms_t(\bmx)}.
\end{equation}
\end{theorem}

\section{Diffusion-PINN Sampler}
\label{section: diffusion pinn sampler}
We consider sampling from a probability density $\pi(\bmx)=\mu(\bmx)/Z$ with $\bmx\in\bR^d$, where $\mu(\bmx)$ has an analytical form and $Z=\int_{\mathbb{R}^d}\mu(\bmx)\mathrm{d}\bmx$ is the intractable normalizing constant. 
Throughout, we only consider the forward process \eqref{eq: diffusion forward} with an explicit conditional density of $\bmx_t|\bmx_0\sim \pi_{t| 0}(\cdot| \bmx_0)$. 
We denote by $\pi_t$ the marginal density of $\bmx_t$ associated with \eqref{eq: diffusion forward} from $\bmx_0\sim\pi_0=\pi$. 

Inspired by diffusion models, sampling can be performed by simulating a reverse process \eqref{eq: diffusion sampling} targeting at $\pi(\bmx)$, given an accurate estimate of the perturbed scores $\bms_{\theta}(\bmx, t)\approx\nx\log\pi_t(\bmx)$,
\begin{equation}\label{eq: diffusion sampling}
\rmd\bmx_t = \left[\bmf(\bmx_t,t) - g^2(t)\bms_{\theta}(\bmx_t, t)\right]\rmd t + g(t)\rmd\Bar{\bmB}_t,\quad \bmx_T\sim \pi_{\textrm{prior}},
\end{equation}
where $\pi_{\textrm{prior}}$ denotes the stationary distribution of the forward process \eqref{eq: diffusion forward} and $T$ is large enough such that $\pi_T\approx \pi_{\textrm{prior}}$.
However, unlike generative models, sampling tasks lack training data from $\pi$, which hinders the application of denoising score matching for perturbed score estimation.
In this section, we propose to solve the log-density FPE \eqref{eq:log-fp} with PINN to estimate the perturbed scores.
While the score FPE can also be used for this purpose, we find that it may fail to learn the mixing proportions properly when the target contains isolated modes.

\subsection{Failure of Score FPE}
\label{subsection: failure of score fp}

Consider the case where the target is a mixture of Gaussians (MoG) with two distant modes.
The following example shows that, for two MoGs with the same modes but different weights, the Fisher divergence\footnote{The Fisher divergence is defined as $F(\nu_1,\nu_2):=\bE_{\bmx\sim\nu_1}\left[\left\|\nx\log\nu_1(\bmx)-\nx\log\nu_2(\bmx)\right\|^2\right]$ for two probability measures $\nu_1$ and $\nu_2$.} between them can be arbitrarily small but the KL divergence between them remains large when the two modes are sufficiently separated.
See Figure \ref{fig:KL_density_score} (left) for an illustration of this phenomenon.
More general theoretical results can be found in Appendix \ref{subsection: appendix mog failure}.
\begin{example}\label{example: mog failure}
For any $\tau>0$, there exists $M_{\tau}(d)>0$ such that the following holds. For every $\bma\in\bR^d$ satisfied $\|\bma\|\ges M_{\tau}(d)$, $w_1, w_2, \tilde{w}_1, \tilde{w}_2\ges 0.1$, $w_1+w_2=1$, and $\tilde{w}_1+\tilde{w}_2=1$, MoG $\pi^M=w_1\cN(\bma,I_d)+w_2\cN(-\bma,I_d)$ and $\tilde{\pi}^M=\tilde{w}_1\cN(\bma,I_d)+\tilde{w}_2\cN(-\bma,I_d)$ satisfy
\begin{equation}\label{eq: mog failure ineq}
\operatorname{KL}(\pi^M\|\tilde{\pi}^M)\ges w_1\log\frac{w_1}{\tilde{w}_1} + w_2\log\frac{w_2}{
\tilde{w}_2} - \tau, \text{ but } F(\pi^M,\tilde{\pi}^M) < \tau,
\end{equation}
where $F(\pi^M,\tilde{\pi}^M)$ denotes the Fisher divergence between $\pi^M$ and $\tilde{\pi}^M$.
\end{example}

\subsubsection{Solving Score FPE Struggles to Learn the Weights}
Let $\pi^M,\tilde{\pi}^M$ be the MoGs in Example \ref{example: mog failure}. For any $t\in [0, T]$, $\pi^M_t$ denotes the marginal distribution of $\bmx_t$ associated with the forward process \eqref{eq: diffusion forward} from $\bmx_0\sim\pi^M$. We denote $\bms_t^M(\bmx):=\nx\log\pi^M_t(\bmx)$ which is the solution to \eqref{eq:score-fp} with $\bms_0^M(\bmx)=\nx\log\pi^M(\bmx)$. 
Similarly, we define $\tilde{\pi}^M_t$ and $\tilde{\bms}_t^M(\bmx)$.

\begin{figure}[t]
    \centering
    \includegraphics[width=\linewidth]{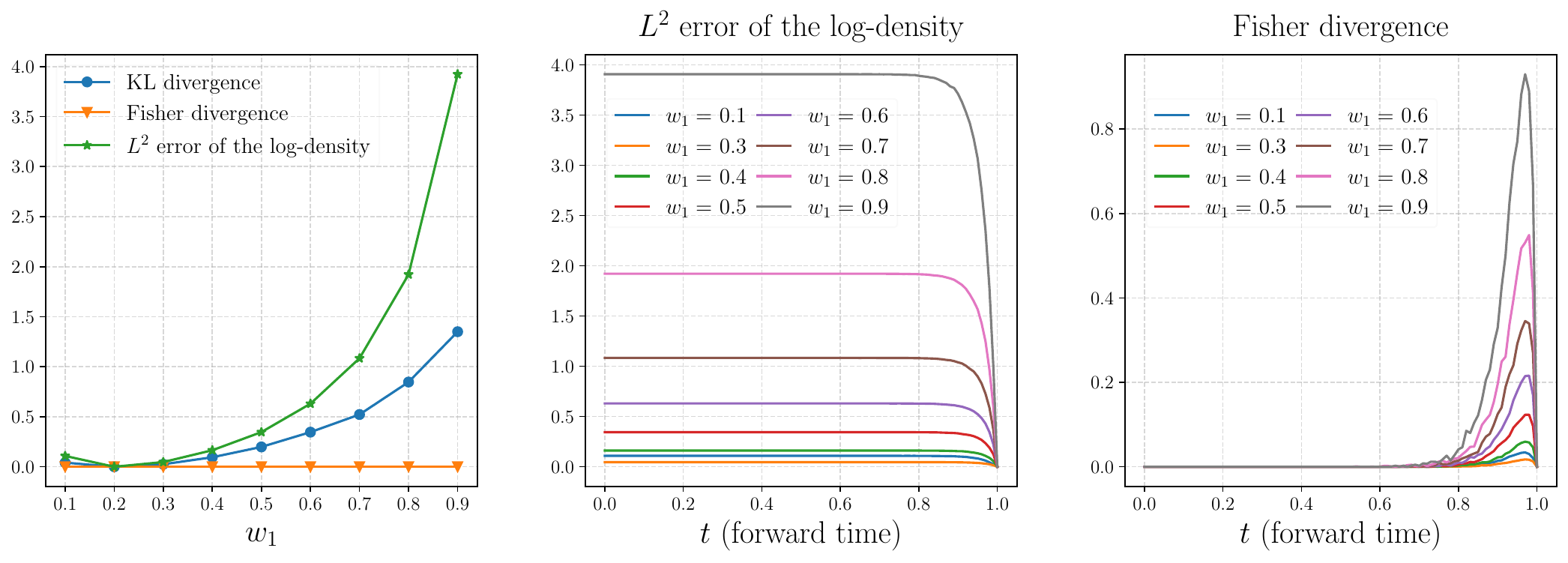}
    \vspace{-0.5cm}
    \caption{\textbf{Left}: KL divergence, Fisher divergence, and log-density error between $\pi^M$ and $\tilde{\pi}^M$ as functions of $w_1$, where $\tilde{w}_1=0.2$ and $\bm{a}=(-5,-5)'$. 
    \textbf{Middle/Right}: The evolution of log-density error/Fisher divergence along the forward process respectively. The forward process achieves standard Gaussian at $t=1$.}
    \label{fig:KL_density_score}

\end{figure}

Consider solving score-FPE \eqref{eq:score-fp} using the following PINN residual loss
\begin{equation}\label{eq: score pinn residual}
\ell_{\textrm{S-res}}\left(\bms;\bmx, t\right):= 
\left\|
\pt\bms_t(\bmx) -
\mathcal{L}_{\textrm{S-FPE}}\bms_t(\bmx) 
\right\|^2,
\end{equation}
Though $\pi^M$ and $\tilde{\pi}^M$ are equipped with different weights, their scores both satisfy the PDE \eqref{eq:score-fp} such that $\ell_{\textrm{S-res}}(\bms^M;\bmx, t)=\ell_{\textrm{S-res}}(\tilde{\bms}^M;\bmx, t)=0$ for any $(\bmx, t)\in\bR^d\times [0, T]$. 
The PINN approach, therefore, can only distinguish $\pi^M$ and $\tilde{\pi}^M$ through the initial condition. 
However, Example \ref{example: mog failure} shows that the difference between $\bms_0^M(\bmx)$ and $\tilde{\bms}_0^M(\bmx)$ can be arbitrarily small, indicating the difficulty of correctly identifying the weights by solving the score FPE.
Figure \ref{fig:KL_density_score} (right) shows that the perturbed score can not tell the difference of weights until the every end of the forward process.
On the other hand, it is noticeable that the perturbed log-density distinguishes the weights well throughout the forward process (Figure \ref{fig:KL_density_score}, middle).
This suggests us to solve log-density FPE and compute the scores by taking the gradient of the approximated log-density.

\subsection{Solving Log-density FPE}
\label{subsection: solving log-density FPE}
To estimate the perturbed scores, we consider solving log-density FPE with initial condition:
\begin{subequations}
\arraycolsep=1.8pt
\def\arraystretch{1.5}
\label{eq: cauchy problem}
\begin{align}
    \pt u_t(\bmx) &= 
\mathcal{L}_{\textrm{L-FPE}}u_t(\bmx), \label{eq: solving cauchy problem 1}\\
    u_0(\bmx)&=\log\mu(\bmx), \label{eq: solving cauchy problem 2}
\end{align}
\end{subequations}
where the exact solution is $u^{*}_t(\bmx)=\log\mu_t(\bmx):=\log\pi_t(\bmx)+\log Z$ (which induces the same score as $\nx u^{*}_t(\bmx)=\nx\log\pi_t(\bmx)$).
In what follows, we describe how to find an approximation $ u_\theta(\bmx, t)$ to $u^*_t(\bmx)$ within the PINN framework.

\paragraph{Target-informed parameterization.}
To incorporate the initial condition \eqref{eq: solving cauchy problem 2}, we use the following parameterization for the log-density function
\begin{equation}\label{eq: parameterized solution}
     u_\theta(\bmx, t)=\frac{T-t}{T}\log \mu(\bmx) + \frac{t}{T}\times \mathrm{NN}_{\theta}(\bmx, t),\quad \forall (\bmx, t)\in \mathbb{R}^d\times [0, T],
\end{equation}
where $\mathrm{NN}_{\theta}(\bmx, t): \mathbb{R}^d\times [0, T]\to \mathbb{R}$ is a deep neural network. 
This parameterization satisfies the initial condition \eqref{eq: solving cauchy problem 2}, thus we only need to consider the PINN residual loss induced by \eqref{eq: solving cauchy problem 1}.
Similar strategy is also used in consistency models \citep{pmlr-v202-song23a}.

\paragraph{Underlying distribution for collocation points.} When training PINN, it is very important to collect proper collocation points $(\bmx_t,t)\in\bR^d\times [0, T]$ where $\bmx_t\sim\nu_t$. We expect samples from $\nu_t$ to cover the high-density domain of $\pi_t$ where PINN can provide a good approximation.
To achieve this, we first generate samples $\bmx_0\sim\nu_0$ by running a short chain of Langevin Monte Carlo (LMC) for $\pi$ so that $\nu_0$ covers the high density domain of $\pi$.
Given $\bmx_0\sim\nu_0$, we obtain $\bmx_t\sim\nu_t$ by sampling from the conditional distribution of the forward process given $\bmx_0$, namely, $\bmx_t|\bmx_0\sim \pi_{t| 0}(\cdot| \bmx_0)$.

\paragraph{Training objective.} 
One useful property of the forward process \eqref{eq: diffusion forward} is that $\bmx_T\sim\pi_T\approx \pi_{\textrm{prior}}$ when $T$ is large. 
In practice, we may use this property to further regularize the PINN residual loss, leading to the following training objective:
\begin{equation}\label{eq: log-fp objective}
\begin{aligned}
L_{\textrm{train}}(u_\theta):= &\ 
\bE_{t\sim\mathcal{U}[0, T]}\bE_{\bmx_t\sim \nu_t}\left[
\beta^2(t)\cdot \left\|\pt  u_\theta(\bmx_t, t) - \mathcal{L}_{\textrm{L-FPE}}  u_\theta(\bmx_t, t)\right\|^2
\right]\\
&\ \quad\quad
+\lambda \cdot
\bE_{\bm{z}\sim \pi_{\textrm{prior}}}
\left[\ell_{\textrm{reg}}(u_\theta; T, \bm{z})\right],
\end{aligned}
\end{equation}
where $\ell_{\textrm{reg}}(u_\theta; T, \bm{z}):=\left\|\nabla_{\bmz} u_\theta(\bm{z},T)-\nabla_{\bmz} \log\pi_{\textrm{prior}}(\bm{z})\right\|^2$ denotes the regularization term, $\beta(t)$ is a weight function and $\lambda$ is a regularization coefficient.
We seek a good approximation $ u_\theta(\bmx, t)$ by minimizing \eqref{eq: log-fp objective} via stochastic optimization methods where the stochastic gradient is computed by Monte Carlo estimation.
Our algorithm is summarized in Algorithm \ref{algorithm: PINN}.

Once $ u_\theta(\bmx, t)$ is learned, the induced score approximation is then substituted into the reverse process \eqref{eq: diffusion sampling}, resulting in a new variant of diffusion-based sampling method that we call \emph{Diffusion-PINN Sampler} (DPS).

\begin{algorithm}[t!]
    \caption{: Solving Log-density FPE via PINN}
    \label{algorithm: PINN}
    \begin{algorithmic}[1]
        \REQUIRE Unnormalized density $\mu(\bmx)$, the number of training iterations $N$, the number of samples used to estimate the training objective \eqref{eq: log-fp objective} $M$, the running time of the forward process \eqref{eq: diffusion forward} $T$.
        \STATE Initialize the parameterized solution $ u_\theta(\bmx, t)$ using target-informed parameterization \eqref{eq: parameterized solution}.
        \FOR{$n=1,\cdots, N$}
        \STATE Sample i.i.d. $t_i\sim \mathcal{U}[0, T], 1\les i\les M$.
        \STATE Sample i.i.d. $\bmx_i^0\sim \nu_0$ and $\bm{z}_i\sim \pi_{\textrm{prior}}, 1\les i\les M$.
        \STATE Sample collocation points by the forward process \eqref{eq: diffusion forward}: 
        $\bmx^{t_i}_{i}\sim \pi_{t_i| 0}(\cdot|\bmx_i^0), 1\les i\les M$.
        \STATE Compute the training objective (\ref{eq: log-fp objective}) by Monte Carlo estimation 
        \begin{equation}\label{eq: monte-carlo estiamtion of objective in algorithm}
            L_{\textrm{MCMC}}(u_\theta):= 
            \frac{1}{M}\sum_{i=1}^M \beta^2(t_i)\cdot \left\|\pt u_\theta(x_i^{t_i},t_i) - 
            \mathcal{L}_{\textrm{L-FPE}} u_\theta(x_i^{t_i},t_i)\right\|^2  + 
            \frac{\lambda}{M}\sum_{i=1}^M 
            \ell_{\textrm{reg}}(u_\theta; T, \bm{z}_i).
        \end{equation}
        \STATE Gradient-based optimization: $\theta\leftarrow \texttt{Optimizer}\left(\theta, \nabla_{\theta}L_{\textrm{MCMC}}(u_\theta)\right)$.
        \ENDFOR
        \RETURN Parameterized solution $ u_\theta(\bmx, t)$.
    \end{algorithmic}
\end{algorithm}

\paragraph{Hutchinson's trick for the gradient of the PINN residual.}
Hutchinson’s trace estimator provides a stochastic method for estimating the trace of any square matrix and is commonly used in Laplacian estimation. 
However, directly using Hutchinson’s trick here can result in biased gradient estimation. 
To address this issue, we propose a novel variant of Hutchinson’s trick that allows unbiased gradient estimation.
Recall that the PINN residual can be decomposed as 
\[
\pt u_\theta - \cL_{\textrm{L-FPE}}u_\theta:=\underbrace{\pt u_\theta + \bmf\cdot\nx u_\theta + \nabla\cdot \bmf -\frac{g^2(t)}{2}\|\nx u_\theta\|^2}_{:=\cL_{\textrm{I}} u_\theta} - \frac{g^2(t)}{2}\Delta u_\theta.
\]
Using this decomposition, the PINN residual loss $\|\pt u_\theta - \cL_{\textrm{L-FPE}} u_\theta\|^2$ has the following gradient,
\[
\begin{aligned}
    & \nabla_\theta \left\|\pt u_\theta - \cL_{\textrm{L-FPE}} u_\theta\right\|^2
 = 2\left(\cL_{\textrm{I}} u_\theta-\frac{g^2(t)}{2}\Delta u_\theta\right)\nabla_\theta\left(\cL_{\textrm{I}} u_\theta-\frac{g^2(t)}{2}\Delta u_\theta\right) \\
 = & 2\left(\cL_{\textrm{I}} u_\theta-\frac{g^2(t)}{2}\cdot\bE_{v_1}\left[v_1^\top \nx\left(v_1^\top \nx u_\theta\right)\right]\right)\nabla_\theta \left(\cL_{\textrm{I}} u_\theta-\frac{g^2(t)}{2}\cdot\bE_{v_2}\left[v_2^\top \nx\left(v_2^\top \nx u_\theta\right)\right]\right) \\
 = & \bE_{v_1,v_2}\left[2\left(\cL_{\textrm{I}} u_\theta-\frac{g^2(t)}{2}\cdot v_1^\top \nx\left(v_1^\top \nx u_\theta\right)\right)\nabla_\theta \left(\cL_{\textrm{I}} u_\theta-\frac{g^2(t)}{2}\cdot v_2^\top \nx\left(v_2^\top \nx u_\theta\right)\right)\right],
\end{aligned}
\]
where $v_1$ and $v_2$ are independent and satisfy $\mathbb{E}_{v_1}[v_1v_1^\top]=\mathbb{E}_{v_2}[v_2v_2^\top]=\boldsymbol{I}_d$.
Therefore, the following objective yields an unbiased gradient estimate of the PINN residual loss,
\[
\mathbb{E}_{v_1,v_2}\left[\textrm{Detach}\left(2\left(\cL_{\textrm{I}} u_\theta-\frac{g^2(t)}{2}\cdot v_1^\top \nx\left(v_1^\top \nx u_\theta\right)\right)\right)\left(\cL_{\textrm{I}} u_\theta-\frac{g^2(t)}{2}\cdot v_2^\top \nx\left(v_2^\top \nx u_\theta\right)\right)\right].
\]

\section{Theoretical Guarantees}

\paragraph{Notations.}
Let us denote $e_t(\bmx):= u_\theta(\bmx, t)-u_t^*(\bmx)$ and $r_t(\bmx):=\pt u_\theta(\bmx, t)-\mathcal{L}_{\textrm{L-FPE}} u_\theta(\bmx, t)$. 
For any $C\in\bR$ and $t\in[0, T]$, we define the weighted PINN objective on $\Omega$ as 
\begin{equation}\label{eq: def of PINN residual objective}
L_{\textrm{PINN}}(t;C):=\int_0^t e^{C(t-s)} \|r_s(\cdot)\|^2_{L^2(\Omega;\nu_s)}\rmd s,
\end{equation}
where $\{\nu_t\}_{t=0}^T$ denotes the underlying distribution for collocation points introduced in Section \ref{subsection: solving log-density FPE} which satisfies the FPE $\pt\nu_t(\bmx)=\mathcal{L}_{\textrm{FPE}}\nu_t(\bmx)$.

\subsection{Approximation Error of PINN for Log-density FPE}
\label{subsection: convergence PINN analysis}
In this section, we provide an upper bound on the approximation error of PINN for the log-density FPE \eqref{eq:log-fp} on a constrained domain $\Omega$. 
Namely, we control $\|e_t(\cdot)\|^2_{L^2(\Omega;\nu_t)}$ and $\|\nx e_t(\cdot)\|^2_{L^2(\Omega;\nu_t)}$ by the residual loss $\|r_t(\cdot)\|^2_{L^2(\Omega;\nu_t)}$ and the weighted PINN objective \eqref{eq: def of PINN residual objective}. 
We make the following assumptions.

\begin{assumption}
\label{PINN assumption: boundary identity}
    $u^*$ and $u_\theta$ are the same on the boundary, i.e., $u^*_t(\bmx)=u_\theta(\bmx,t)$ on $\partial\Omega\times [0, T]$.
\end{assumption}

\begin{assumption}
\label{PINN assumption: diffusion bound}
    For any $t\in \left[0, T\right]$, $g^2(t)$ is bounded: $m_1\les g^2(t)\les M_1$ for some $m_1, M_1>0$.
\end{assumption}

\begin{assumption}
\label{PINN assumption: continuous assumption}
$\log \nu_t(\bmx), u^*_t(\bmx),  u_\theta(\bmx, t)\in \cC^2(\Omega\times [0,T])$.
\end{assumption}

Assumption \ref{PINN assumption: boundary identity} is necessary for us to ensure the uniqueness of the solution to \eqref{eq:log-fp} on $\Omega$, which is also considered in \citet{deveney2023closing,wang20222}.
Assumption \ref{PINN assumption: diffusion bound}, \ref{PINN assumption: continuous assumption} are also considered in \citet{deveney2023closing}.
Based on Assumption \ref{PINN assumption: continuous assumption}, there exists $\Bv_0, \Bu_0, \hB_0, \Bv_1, \Bu_1, \hB_1\in\bR_+$ and $\Bv_2, \Bu_2, \hB_2\in \bR$ depended on $\Omega$ such that for any $(\bmx, t)\in\Omega\times [0, T]$, we have
\[\begin{aligned}
   & \vert\pt \log \nu_t(\bmx)\vert\les \Bv_0, \quad  
   \vert\pt u^*_t(\bmx)\vert \les \Bu_0, \quad
   \vert \pt  u_\theta(\bmx, t)\vert \les \hB_0, \\
   & \|\nx\log \nu_t(\bmx)\|^2\les \Bv_1, \quad
   \|\nx u^*_t(\bmx)\|^2\les \Bu_1, \quad
   \|\nx  u_\theta(\bmx, t)\|^2\les \hB_1, \\
   & \Lap\log \nu_t(\bmx)\les \Bv_2, \quad
   \Lap u^*_t(\bmx)\ges \Bu_2, \quad 
   \Lap  u_\theta(\bmx, t)\ges \hB_2.
\end{aligned}\]
In practice, using weights clipping strategy as in \citet{pmlr-v70-arjovsky17a}, we can control the regularity of neural network approximation $ u_\theta(\bmx, t)$, thus bound the constants $\hB_0, \hB_1,\hB_2$.

We summarize our main results in the following theorem.
The proof is deferred to Appendix \ref{appendix: Theorem PINN analysis proof}, which generalizes the framework in \citet{deveney2023closing}.

\begin{theorem}
\label{thm: PINN analysis}
Suppose that Assumption \ref{PINN assumption: boundary identity}, \ref{PINN assumption: diffusion bound}, and \ref{PINN assumption: continuous assumption} hold. 
We further assume that $u_\theta(\bmx,0)=u^*_0(\bmx)$\footnote{This is a reasonable assumption due to the target-informed parameterization introduced in Section \ref{subsection: solving log-density FPE}.} for any $\bmx\in\Omega$.
Then for any positive constant $\vareps>0$, the following holds for any $0\les t\les T$,
\begin{equation}\label{eq: thm PINN analysis e 1 in context}
\|e_t(\cdot)\|^2_{L^2(\Omega;\nu_t)}\les \vareps L_{\emph{\textrm{PINN}}}(t;C_1(\vareps)).
\end{equation}
Moreover, for any $0\les t\les T$,
\begin{equation}\label{eq: thm PINN analysis nabla e 1 in context}
m_1\|\nx e_t(\cdot)\|^2_{L^2(\Omega;\nu_t)}\les
\vareps\|r_t(\cdot)\|^2_{L^2(\Omega;\nu_t)}
+C_3(\vareps)L_{\emph{\textrm{PINN}}}(t;C_1(\vareps)) 
+C_2\sqrt{\vareps L_{\emph{\textrm{PINN}}}(t;C_1(\vareps))}.
\end{equation}
where $C_2:=2\sqrt{2}(\hB_0^2+B_0^{*2})^{1/2}$, $C_3(\vareps):=\vareps(C_1(\vareps)+\Bv_0)$, and $C_1(\vareps)$ is a constant depended on $\Bv_1$, $\Bu_1$, $\hB_1$, $\Bv_2$, $\Bu_2$, $\hB_2$ and $m_1$, $M_1$.
\end{theorem}
\begin{remark}
The results of \citet{wang20222} show that the 
$L^2$-error cannot be bounded by the PINN residual with universal constants independent of the approximate solution.
Therefore, some natural continuity assumptions (Assumption \ref{PINN assumption: continuous assumption}) about the approximate solution are necessary to control the $L^2$-error by the PINN residual. 
It is noted that this continuity assumption can be satisfied by regularizing the neural network via weight clipping \citep{pmlr-v70-arjovsky17a}, and would not sacrifice much approximation accuracy as the true solution is initialized as the log-density of the target and follows the diffusion process (e.g., the OU process) that would only become smoother as time evolves. 
Moreover, our upper bound of $L^2$-error depends on continuity constants rather than an universal bound. 
In this regard, our analysis aligns with the results of \citet{wang20222}, but with a more flexible bound based on some natural continuity assumption in the context of diffusion-based sampling.
\end{remark}

\subsection{Convergence of Diffusion-PINN Sampler}
\label{subsection: convergence Sampler analysis}
In this section, we present our convergence analysis of DPS based on Theorem \ref{thm: PINN analysis} and the analysis of score-based generative modeling in \cite{chen2023improved}. 
Following \citet{chen2023sampling,chen2023improved}, we focus on the forward process with $\bmf(\bmx,t)=-\frac{1}{2}\bmx$ and $g(t)\equiv 1$, which is driven by
\begin{equation}\label{eq: special forward process}
\rmd\bmx_t=-\frac{1}{2}\bmx_t\rmd t+ \rmd\bmB_t, \quad \bmx_0\sim \pi, \quad 0\les t\les T.
\end{equation}
In practice, we use a discrete-time approximation for the reverse process. 
Let $0=t_0 < \cdots < t_N=T$ be the discretization points and $h_k:=t_k-t_{k-1}$ be the step size for $1\les k\les N$. 
Let $t_k^{\prime}:=T-t_{N-k}$ for $0\les k\les N$ be the corresponding discretization points in the reverse SDE. 
In our analysis, we consider the exponential integrator scheme which leads to the following sampling dynamics for $0\les k\les N-1$,
\begin{equation}\label{eq: special reverse process with discretization}
    \rmd \widehat{\bm{y}}_t = \left(\frac{1}{2}\widehat{\bm{y}}_t+\bms_{T-t_k^{\prime}}(\widehat{\bm{y}}_{t_k^{\prime}})\right)\rmd t + \rmd\bmB_t, \quad \widehat{\bm{y}}_0\sim \mathcal{N}(\bm{0},\bm{I}_d),\quad t\in [t_k^{\prime}, t_{k+1}^{\prime}],
\end{equation}
where $\bms_t(\bmx)\approx\nx\log\pi_t(\bmx)$ denotes the score approximation.
Let $\widehat{\pi}_T$ denotes the distribution of $\widehat{\bm{y}}_T$ from \eqref{eq: special reverse process with discretization}. 
We summarize all the assumptions we need as follows.

\begin{assumption}\label{Sampler assumption: target distr}
The target distribution admits a density $\pi\in\mathcal{C}^2(\mathbb{R}^d)$ where $\nx\log\pi(\bmx)$ is $K$-Lipschitz and has the finite second moment, i.e., $M_2:=\mathbb{E}_{\pi}\left[\|\bmx\|^{2}\right]<\infty$.
\end{assumption}

\begin{assumption}\label{Sampler assumption: bounded region}
For any $\delta>0$, there exists bounded $\Omega$ such that $\int_{\Omega^c}\pi_t(\bmx)\|\nx\log \pi_t(\bmx)\|^2\rmd\bmx\les\delta$ for any $t\in[0, T]$.
\end{assumption}

\begin{assumption}\label{Sampler assumption: base target ratio}
For any $(\bmx, t)\in \Omega\times [0, T]$, there exists $R_t\ges 0$ depended on $t$, such that $\frac{\pi_t(\bmx)}{\nu_t(\bmx)}\les R_t$.
\end{assumption}

Theorem \ref{thm: Diffusion PINN Sampler analysis} summarizes our main theoretical results of DPS.
The proof can be found in Appendix \ref{appendix: Sampler analysis proof}, which is based on the convergence results of score-based generative modeling in \citet{chen2023improved} and our Theorem \ref{thm: PINN analysis}.

\begin{theorem}\label{thm: Diffusion PINN Sampler analysis}
Suppose that $T\ges 1$, $K\ges 2$, and Assumptions \ref{PINN assumption: boundary identity}-\ref{Sampler assumption: base target ratio} hold. 
For any $\delta>0$, let $\Omega$ be chosen as in Assumption \ref{Sampler assumption: bounded region}.
For any positive constant $\vareps>0$, we further assume that $ u_\theta(\bmx, t)$ satisfies
\begin{equation}\label{eq: Sampler analysis PINN loss condition}
\vareps\sum_{k=1}^N h_kR_{t_k}\|r_{t_k}(\cdot)\|^2_{L^2(\Omega;\nu_{t_k})}\les\delta_1
\quad{\textrm{and}}
\quad\vareps\sum_{k=1}^N h_kR_{t_k}L_{\emph{\textrm{PINN}}}(t_k;C_1(\vareps))\les \delta_2.
\end{equation} 
Then there is a universal constant $\alpha\ges 2$ such that, using step size
$h_k=h\min\{\max\{t_k,1/(4K)\}, 1\}$, $0<h\les 1/(\alpha d)$, and $\bms_t(\bmx)=\nx u_\theta(\bmx, t)\cdot\bbmone\{\bmx\in\Omega\}$ in \eqref{eq: special reverse process with discretization}, we have the following upper bound on the KL divergence between the target and the approximate distribution
\begin{equation}\label{eq: main results Diffusion-PINN sampler 1}
\operatorname{KL}(\pi\|\widehat{\pi}_T)\lesssim 
 (d+M_2)\cdot e^{-T} 
 + d^2h(\log K+T) 
 + T\delta 
 + \delta_1 
 + C_5(\vareps)\delta_2
 + C_2\sqrt{\sum_{k=1}^Nh_kR_{t_k}\delta_2},
\end{equation}
where $C_5(\vareps):=C_1(\vareps)+\Bv_0$, $C_2$ and $C_1(\vareps)$ are defined in Theorem \ref{thm: PINN analysis}.
\end{theorem}

\subsection{Theoretical Comparison between Different Sampling Methods for Collocation Points}

In practice, we typically lack prior knowledge of the high-probability regions of the diffusion path starting from the target distribution. As a result, specifying a sufficiently large support for uniform sampling of collocation points, becomes challenging and inefficient, especially in high-dimensional settings.
In contrast, we employ a more sophisticated strategy for generating collocation points that integrates Langevin Monte Carlo (LMC) with the forward pass (see Section \ref{subsection: solving log-density FPE} for details). 
In this section, we theoretically demonstrate our strategy for collocation points generation has a faster convergence for DPS.

Similar to Theorem \ref{thm: PINN analysis} and \ref{thm: Diffusion PINN Sampler analysis}, we can establish theoretical guarantees for uniformly sampled collocation points. 
However, these guarantees differ from those in Theorem \ref{thm: PINN analysis} and \ref{thm: Diffusion PINN Sampler analysis}, since the distribution of uniformly sampled collocation points does not depend on $t$ and does not satisfy the corresponding FPE.
This distinction necessitates a slightly different analysis, requiring an additional assumption compared to the previous analyses (see more details in Section \ref{subsubsect: approximation error extended}).
Our findings are summarized in Theorem \ref{thm: extended PINN analysis for uniform sampling} and \ref{thm: Unif KL bound}, with the corresponding proofs provided in Appendix \ref{appendix: proof of thm: extended PINN for uniformed sampling} and Appendix \ref{appendix: proof of thm: Unif KL bound}.

\subsubsection{Approximation Error of PINN for Log-density FPE with Uniformly Sampled Collocation Points}
\label{subsubsect: approximation error extended}
In this section, we control the approximation error of PINN for the log-density FPE \eqref{eq:log-fp} on a constrained domain $\Omega$ with uniformly sampled collocation points $\nu_t\sim{\rm Unif}(\Omega)$.
First, we redefine the weighted PINN objective on $\Omega$ as follows and state an additional assumption for our analysis in Assumption \ref{assump:comparsion-add_uniform},
\[L_{\textrm{PINN}}^{\textrm{Unif}}(t;C):=\int_0^t e^{C(t-s)} \left\|r_s(\cdot)\right\|^2_{L^2(\Omega)}\rd s.\]
\begin{assumption}\label{assump:comparsion-add_uniform}
    For any $(\bmx, t)\in \Omega\times [0, T]$, $\nabla\cdot \bmf(\bmx,t)\les \mathrm{m}_2$ for some $\mathrm{m}_2\in\bR$.
\end{assumption}
\begin{theorem}
\label{thm: extended PINN analysis for uniform sampling}
Suppose that Assumption \ref{PINN assumption: boundary identity}, \ref{PINN assumption: diffusion bound}, \ref{PINN assumption: continuous assumption} and \ref{assump:comparsion-add_uniform} hold.
We further assume that $u_\theta(\bmx,0)=u^*_0(\bmx)$ for any $\bmx\in\Omega$. 
Then for any positive constant $\eps>0$, the following holds for any $t\in[0, T]$,
\begin{equation}\label{eq: Unif thm PINN analysis e 1}
    \|e_t(\cdot)\|^2_{L^2(\Omega)}\les\eps L_{\emph{\textrm{PINN}}}^{\emph{\textrm{Unif}}}(t;C_1^{\emph{\textrm{U}}}(\eps)).
\end{equation}
Moreover, for any $t\in[0, T]$,
\begin{equation}\label{eq: Unif thm PINN analysis nabla e 1}
    m_1\|\nx e_t(\cdot)\|^2_{L^2(\Omega)}\les 
    \eps\|r_t(\cdot)\|^2_{L^2(\Omega)}
    +\eps\cdot C_1^{\emph{\textrm{U}}}(\eps)L_{\emph{\textrm{PINN}}}^{\emph{\textrm{Unif}}}(t;C_1^{\emph{\textrm{U}}}(\eps))
    + C_2\sqrt{\eps L_{\emph{\textrm{PINN}}}^{\emph{\textrm{Unif}}}(t;C_1^{\emph{\textrm{U}}}(\eps))},
\end{equation}
where $C_1^{\emph{\textrm{U}}}(\eps):=\frac{1}{\eps}+\mathrm{m}_2-\frac{m_1}{2}\left(B_2^*+\widehat{B}_2\right)$ and $C_2$ is defined in Theorem \ref{thm: PINN analysis}.
\end{theorem}
\begin{remark}
    Indeed, we only need the assumption of the lower bound on $g^2(t)$ from Assumption \ref{PINN assumption: diffusion bound} and the continuity assumptions for $u_t^*(\bmx)$ and $u_\theta(\bmx,t)$ from Assumption \ref{PINN assumption: continuous assumption} in Theorem \ref{thm: extended PINN analysis for uniform sampling}.
\end{remark}

\subsubsection{Convergence of Diffusion-PINN Sampler with Uniformly Sampled Collocation Points}
In this section, within the similar setting in section \ref{subsection: convergence Sampler analysis}, we present the convergence analysis of DPS when the collocation points are uniformly sampled from $\Omega$.
\begin{theorem}\label{thm: Unif KL bound}
    Suppose that $T\ges 1, K\ges 2$, and Assumption \ref{PINN assumption: boundary identity}, \ref{PINN assumption: diffusion bound}, \ref{PINN assumption: continuous assumption}, \ref{Sampler assumption: target distr}, \ref{Sampler assumption: bounded region}, and \ref{assump:comparsion-add_uniform} hold. 
    For any $\delta>0$, let $\Omega$ be chosen as in Assumption \ref{Sampler assumption: bounded region}. 
    For any positive constant $\eps>0$, we further assume that $  u_\theta(\bmx, t)$ satisfies the following\footnote{Here, we contain the term $\textrm{Vol}(\Omega)$ since the PINN residual objective used for uniform collocation points is given by $\|r_t(\cdot)\|^2_{L^2(\Omega)}/\textrm{Vol}(\Omega)$.},
    \begin{equation}\label{Unif PINN error in KL}
    \begin{aligned}
        &\ \eps\sum_{k=1}^Nh_k\max_{\bmx\in\Omega}\left\{\pi_{t_k}(\bmx)\right\}\cdot\left\|r_{t_k}(\cdot)\right\|^2_{L^2(\Omega)}\les\delta_1\cdot\emph{\textrm{Vol}}(\Omega), \\
        &\ \eps\sum_{k=1}^Nh_k\max_{\bmx\in\Omega}\left\{\pi_{t_k}(\bmx)\right\}\cdot L_{\emph{\textrm{PINN}}}^{\emph{\textrm{Unif}}}(t_k;C_1^{\emph{\textrm{U}}}(\eps))\les\delta_2\cdot\emph{\textrm{Vol}}(\Omega).
    \end{aligned}
    \end{equation}
    Then there is a universal constant $\alpha\ges 2$ such that, using step size $h_k:=h\min\{\max\{t_k,\frac{1}{4K}\}\}$ for $0<h\les \frac{1}{\alpha d}$, and $\bms_t(\bmx)=\nx  u_\theta(\bmx, t)\cdot\1\{\bmx\in\Omega\}$,
    we have the following upper bound on the KL divergence between the target and the approximate distribution,
    \[
    \begin{aligned}
        \operatorname{KL}\left(\pi\|\widehat{\pi}_T\right)\lesssim &\ \left(d+M_2\right)\cdot e^{-T} + d^2h\left(\log K+T\right) + T\delta + \left(\delta_1 + C_1^{\emph{\textrm{U}}}(\eps)\delta_2\right)\cdot\emph{\textrm{Vol}}(\Omega) \\
        &\ \quad + C_2 \sqrt{\sum_{k=1}^Nh_k\max_{\bmx\in\Omega}\left\{\pi_{t_k}(\bmx)\right\}\cdot\delta_2\cdot \emph{\textrm{Vol}}(\Omega)}.
    \end{aligned}
    \]
    where $C_1^{\emph{\textrm{U}}}(\eps)$ and $C_2$ are defined in Theorem \ref{thm: extended PINN analysis for uniform sampling}.
\end{theorem}

Specifically, our results indicate that employing LMC and the forward pass for sampling collocation points is advantageous over uniform sampling. This is because, in the uniform case, the KL bound includes a factor proportional to the volume of the support, ${\rm Vol}(\Omega)$, which can be prohibitively large in high dimensions. In contrast, our bound depends on the density ratio $\pi_t/\nu_t$, which is more manageable due to LMC and converges to $1$ as $t$ increases, thanks to the forward process.

\section{Numerical Experiments}
In this section, we conduct experiments on various sampling tasks to demonstrate the effectiveness and efficiency of the Diffusion-PINN Sampler (DPS) compared to previous methods.
Our sampling tasks includes 9-Gaussians ($d=2$), Rings ($d=2$), Funnel \citep{neal2003slice} ($d=10$), and Double-well ($d=30$), which are commonly used to evaluate diffusion-based sampling algorithms \citep{zhang2021path,berner2022optimal,grenioux2024stochastic}.
For multimodal distributions, the modes are designed to be well-separated, with challenging mixing proportions between different modes (see more details in Appendix \ref{appendix: details on targets}).
For DPS, we employ a time-rescaled forward process and use a weight function $\beta(t)=2(1-t)$ for the PINN residual loss to improve numerical stability.
To generate collocation points for each task, we run a short chain of LMC with a relatively large step size for better coverage of the high-density domain.
For 9-Gaussians, Rings, and Double-well, the PINN residual loss alone suffices for good performance, so we set the regularization coefficient $\lambda=0$.
For Funnel, however, regularization proves helpful, and we set $\lambda=1$ (details in Section \ref{sec:ablation}).
More details on experiment settings can be found in Appendix \ref{appendix: DPS}.

\begin{wrapfigure}[12]{r}{0.4\textwidth}
\vspace{-0.5cm}
\includegraphics[width=\linewidth]{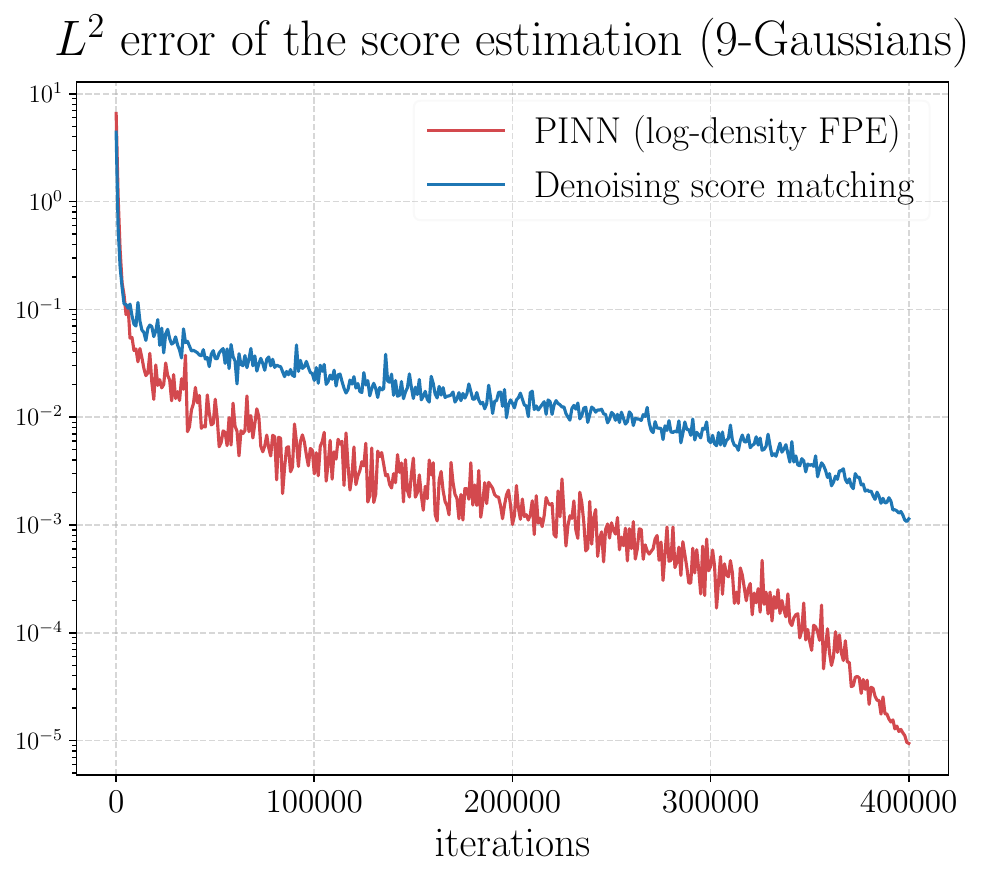}
\vspace{-0.8cm}
\caption{Comparison between solving log-density FPE by PINN and denoising score matching on score estimation.}
\label{fig:score estimation mog-3}
\end{wrapfigure}

\paragraph{Baselines.}
We benchmark DPS performance against a wide range of strong baseline methods. 
For MCMC methods, we consider the Langevin Monte Carlo (LMC).
As for sampling methods using reverse diffusion, we include RDMC \citep{huang2023reverse} and SLIPS \citep{grenioux2024stochastic}.
We also compare with the VI-based PIS \citep{zhang2021path} and DIS \citep{berner2022optimal}.
See Appendix \ref{appendix: details on baselines} for more details.

\subsection{Score Estimation}
We first evaluate the accuracy of score function estimates obtained by solving the log-density FPE (Algorithm \ref{algorithm: PINN}).
To do that, we conduct an experiment on the 9-Gaussians target $\pi$ where we know the ground truth scores throughout the entire forward process.
Figure \ref{fig:score estimation mog-3} shows the $L^2(\pi)$ error of the score estimation for our method compared to denoising score matching \citep{vincent2011connection,song2020denoising}. 
We see clearly that our method provides more accurate score estimation than denoising score matching.

\begin{figure}[t]
    \centering
    \includegraphics[width=1.0\linewidth]{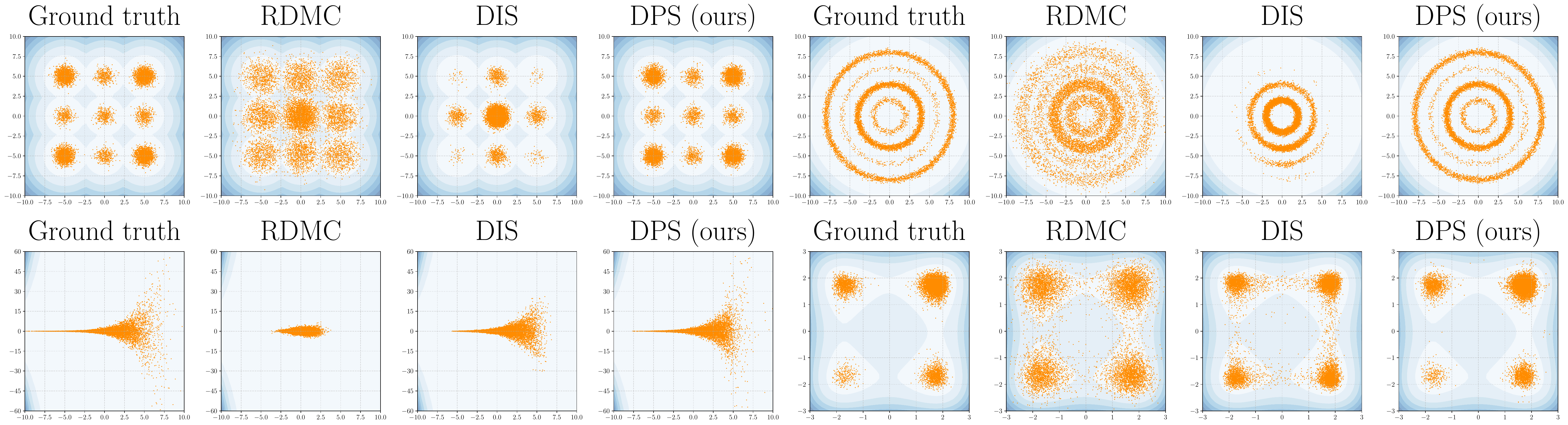}
    \vspace{-0.4cm}
    \caption{Sampling performance of different methods for 9-Gaussians ($d=2$), Rings ($d=2$), Funnel ($d=10$), and Double-well ($d=30$).}
    \label{fig:samples visualization of different methods}
\end{figure}

\begin{table}[t]
    \caption{KL divergence ($\downarrow$) to the ground truth obtained by different methods. Bold font indicates the best results. We use the KL divergence of the first two dimensions for Funnel ($d=10$) and the KL divergence of the first five dimensions for Double-well ($d=30$). All the KL divergence is computed by the ITE package \citep{szabo2014information}.}
    \label{table:KL}
    \centering
    \scalebox{0.8}{
    \begin{tabular}{c  c  c  c  c  c  c}\toprule
          Target
         & LMC & RDMC & SLIPS & PIS & DIS & DPS (ours)
         \\
         \midrule
         9-Gaussians      &
        $1.6568_{\pm 0.0189}$   &
          $1.0844_{\pm 0.0132}$   &
           $0.0901_{\pm 0.0071}$  &
           $2.0042_{\pm 0.0203}$  &
            $2.2758_{\pm 0.0240}$ &
         $\mathbf{0.0131_{\pm 0.0093}}$
         \\ 
         Rings   &
            $2.4754_{\pm 0.0302}$ &
            $0.7487_{\pm 0.0073}$ &
            $0.4127_{\pm 0.0144}$ &
            $2.6985_{\pm 0.0290}$ &
             $2.3433_{\pm 0.0275}$ &
           $\mathbf{0.0176_{\pm 0.0059}}$
         \\ 
         Funnel    &
           $0.1908_{\pm 0.0156}$  &
            $2.0250_{\pm 0.0364}$ &
            $0.1971_{\pm 0.0133}$ &
            $0.4377_{\pm 0.0199}$ &
             $0.2383_{\pm 0.0169}$ &
          $\mathbf{0.0846_{\pm 0.0122}}$
         \\ 
         Double-well  &
            $0.1915_{\pm 0.0122}$ &
            $1.5735_{\pm 0.0162}$ &
            $0.4840_{\pm 0.0145}$ &
            $0.0969_{\pm 0.0114}$ &
            $0.6796_{\pm 0.0139}$ &
           $\mathbf{0.0273_{\pm 0.0113}}$ 
        \\ 
        \bottomrule
    \end{tabular}
    }
\end{table}

\begin{table}[h!]
    \caption{$L^2$ error ($\downarrow$) of the mixing proportions estimation when sampling multimodal target distributions using different methods. Bold font indicates the best results. All the estimation is computed with 1,000 samples.}
    \label{table:weights error}
    \centering
    \scalebox{0.8}{
    \begin{tabular}{c  c  c  c  c  c  c}\toprule
          Target
           & LMC & RDMC & SLIPS & PIS & DIS & DPS (ours)
         \\
         \midrule
         9-Gaussians    &
          $0.5199_{\pm 0.0159}$   &
            $0.1313_{\pm 0.0099}$ &
           $0.0018_{\pm 0.0005}$  &
           $0.4893_{\pm 0.0110}$  &
            $0.7268_{\pm 0.0146}$ &
          $\mathbf{0.0006_{\pm 0.0003}}$ 
         \\  
         Rings   &
           $0.6005_{\pm 0.0251}$  &
            $0.0537_{\pm 0.0035}$ &
            $0.2471_{\pm 0.0144}$ &
            $0.8016_{\pm 0.0194}$ &
             $0.5233_{\pm 0.0194}$ &
          $\mathbf{0.0006_{\pm 0.0006}}$ 
         \\ 
         Double-well   &
            $0.0673_{\pm 0.0082}$ &
            $0.2154_{\pm 0.0075}$ &
            $0.1645_{\pm 0.0113}$ &
            $0.0044_{\pm 0.0011}$ &
            $0.0684_{\pm 0.0035}$ &
           $\mathbf{0.0004_{\pm 0.0002}}$
        \\ 
        \bottomrule
    \end{tabular}
    }
    \vspace{-0.3cm}
\end{table}

\subsection{Sample Quality}

In this section, we compare DPS with the aforementioned baseline methods on various target distributions.
We use KL divergence to evaluate the quality of samples provided by different methods in low dimensional problems (9-Gaussians, Rings), and use the projected KL divergence instead for Funnel and Double-well that are problems with relatively higher dimensions. 
The results are reported in Table \ref{table:KL}.
Figure \ref{fig:samples visualization of different methods} visualizes the samples from different methods.
We clearly see that DPS provides the best approximation accuracy and sample quality among all methods.
Although we use LMC to generate collocation points, DPS greatly outperforms LMC, indicating the power of diffusion-based sampling methods with learned score functions.

For multimodal distributions, we estimate the mixing proportions for different modes using samples generated by different methods, and evaluate the estimation accuracy in terms of $L^2$ error to the true weights.
The results are shown in Table \ref{table:weights error}. 
It is clear that DPS provides accurate weights estimation while other baselines tend to struggle to learn the weights.

\subsection{Ablation Study}\label{sec:ablation}
In this section, we compare the performance of score estimation between solving the score FPE and the log-density FPE, and investigate the effect of regularization in DPS.

We first solve the corresponding score FPE and log-density FPE for a MoG with two distant modes: $\pi^M=0.2\mathcal{N}((-5,-5)^{\prime}, \bm{I}_2)+0.8\mathcal{N}((5,5)^{\prime}, \bm{I}_2)$. 
The left plot in Figure \ref{fig:ablation on failure of score fpe and adaptive sampling} show the PINN residual loss and the score estimation error as functions of the number of iterations.
We see that for the score FPE, the score approximation error decreases rapidly at first but quickly levels off, while the PINN residual loss continues to decrease with more iterations.
In contrast, when solving the log-density FPE, the PINN residual loss and the score approximation error decrease consistently, resulting in more accurate score approximation overall.
The middle and right plots in Figure \ref{fig:ablation on failure of score fpe and adaptive sampling} display the histogram based on samples generated from the reverse SDE using the score estimates from both methods, together with the true marginal density. We observe that the score FPE-based method fails to identify the correct mixing proportions, whereas the log-density FPE-based method successfully recovers the correct weights.

Next, we solve the log-density PFE with different regularization coefficients $\lambda$ on the Funnel target. 
Figure \ref{fig:ablation on regularization} (left) shows the KL divergence for various $\lambda$ as a function of the number of iterations.
We see that, compared to the non-regularized case ($\lambda =0$), both the convergence speed and overall approximation accuracy have been greatly improved when regularization is applied.
The middle and right plots in Figure \ref{fig:ablation on regularization} show the samples generated from DPS with $\lambda=1$ and $\lambda=0$ respectively.
With regularization, DPS provides a better fit to the target distribution, more accurately capturing the thickness in the tails. This indicates that regularization could be beneficial for heavy-tail distributions.

\begin{figure}[t]
    \centering
    \includegraphics[width=\linewidth]{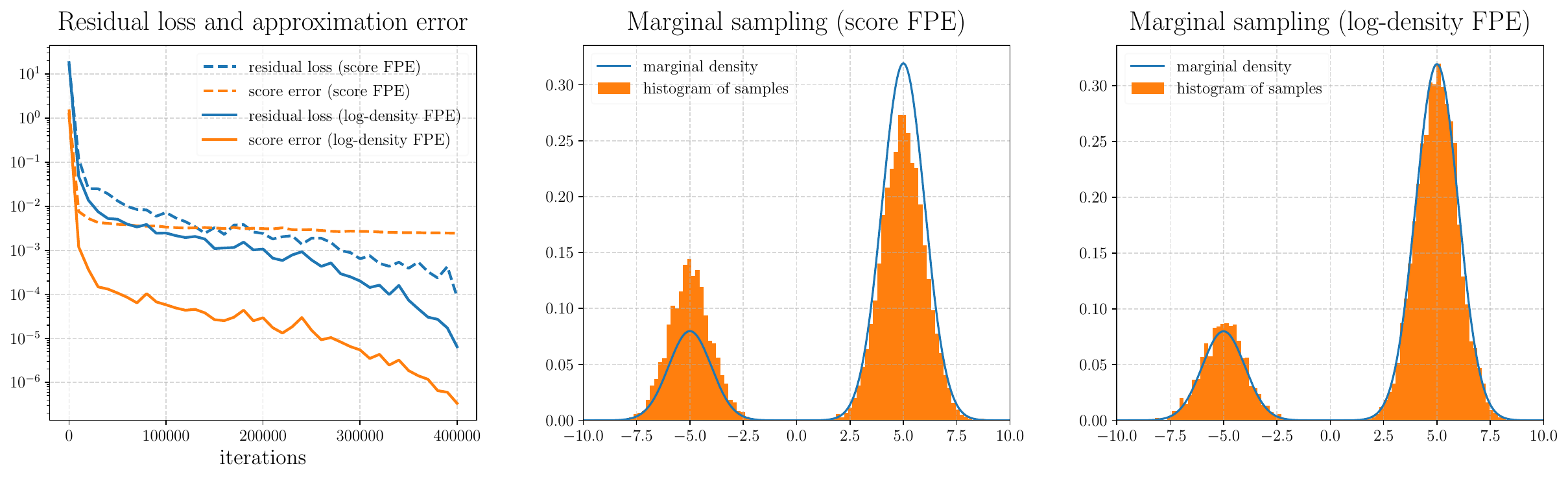}
    \vspace{-0.5cm}
    \caption{\textbf{Left}: PINN residual loss and score approximation error during solving score/log-density FPE; \textbf{Middle/Right}: Marginals of the first dimension from DPS by solving score/log-density FPE for MoG with two modes.}
    \label{fig:ablation on failure of score fpe and adaptive sampling}
\end{figure}

\begin{figure}[t]
    \centering
    \includegraphics[width=\linewidth]{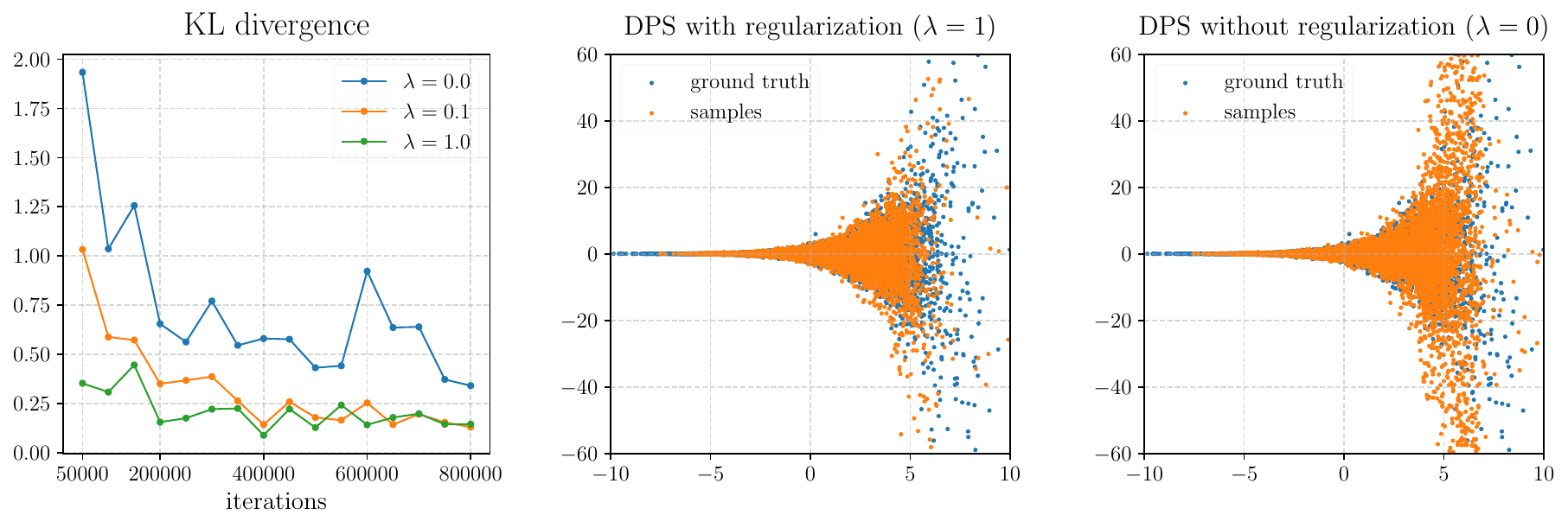}
    \vspace{-0.5cm}
    \caption{\textbf{Left}:  KL divergence to the ground truth during solving log-density FPE with different regularization for Funnel. \textbf{Middle/Right}: Sampling performance of DPS with/without regularization for Funnel.}
    \label{fig:ablation on regularization}
\end{figure}

\section{Conclusion}

In this work, we proposed \emph{Diffusion-PINN Sampler} (DPS), a novel method that leverages Physics-Informed Neural Networks (PINN) and diffusion models for accurate sampling from complex target distributions.
By solving the log-density FPE that governs the evolution of the log-density of the underlying SDE marginals via PINN, DPS demonstrates accurate sampling capabilities even for distributions with multiple modes or heavy tails, and it excels in identifying mixing proportions when the target features isolated modes.
The control of log-density estimation error via PINN residual loss ensures convergence guarantees to the target distribution, building upon established results for score-based diffusion models.
We demonstrated the effectiveness of our approach on multiple numerical examples.
Limitations are discussed in Appendix \ref{appendix: limitations}.

\section*{Acknowledgements}
This work was supported by National Natural Science Foundation of China (grant no. 12201014 and grant no. 12292983).
The research of Cheng Zhang was support in part by National Engineering Laboratory for Big Data Analysis and Applications, the Key Laboratory of Mathematics and Its Applications (LMAM) and the Key Laboratory of Mathematical Economics and Quantitative Finance (LMEQF) of Peking University.
Zhekun Shi is partially supported by the elite undergraduate training program of School of Mathematical Sciences in Peking University.

\bibliographystyle{nameyear}
\bibliography{bibliography}


\appendix

\section{Proofs}

\subsection{Proof of Theorem \ref{thm: fp}}
\label{subsection: proof of log-density/score fpe}

\begin{proof}[Proof of Theorem \ref{thm: fp}]
Recall that $p_t(\bmx)$ denotes the marginal density of $\bmx_t$ following the forward process \eqref{eq: diffusion forward},
and satisfies
\begin{equation}
\label{eq: proof-thm-1-fp}
    \pt p_t(\bmx) = \frac{1}{2}g^2(t)\Lap p_t(\bmx)-\diver\left[
    \bmf(\bmx, t)p_t(\bmx)\right].
\end{equation}
Therefore, the log-density $u_t(\bmx):=\log p_t(\bmx)$ satisfies 
\begin{equation}
\label{eq: proof-thm-1-log-fp}
\pt u_t(\bmx)= \frac{\pt p_t(\bmx)}{p_t(\bmx)}
=\frac{1}{2}g^2(t)\frac{\Lap p_t(\bmx)}{p_t(\bmx)}
-\frac{\diver\left[\bmf(\bmx, t)p_t(\bmx)\right]}{p_t(\bmx)}.
\end{equation}
Note that we have the identities
\begin{equation}
\label{eq: proof-thm-1-trans}
\begin{aligned}
    & \Lap p_t(\bmx) = \diver\left[p_t(\bmx)\nx u_t(\bmx)\right]
    =\nx p_t(\bmx)\cdot\nx u_t(\bmx)
    +p_t(\bmx)\Lap u_t(\bmx),\\
    & \diver\left[\bmf(\bmx, t)p_t(\bmx)\right]=\nx p_t(\bmx)\cdot \bmf(\bmx, t)+p_t(\bmx)\left[\diver\bmf(\bmx, t)\right].
\end{aligned}
\end{equation}
Plug (\ref{eq: proof-thm-1-trans}) into (\ref{eq: proof-thm-1-log-fp}), we have \begin{equation*}
\pt u_t(\bmx)
=  \frac{1}{2}g^2(t)\Lap u_t(\bmx) 
+ \frac{1}{2}g^2(t)\left\|\nx u_t(\bmx)\right\|^2
- \bmf(\bmx, t)\cdot\nx u_t(\bmx) 
- \diver \bmf(\bmx, t).
\end{equation*}
Since $\log p_t(\bmx)$ is sufficiently smooth, we can swap the order of differentiation and get \[\pt \bms_t(\bmx)=\pt\nx u_t(\bmx)=\nx\pt u_t(\bmx).\] 
Hence, the theorem is proved.
\end{proof}

\subsection{Omitted Proof in Example \ref{example: mog failure}}
\label{subsection: appendix mog failure}

\paragraph{Notations.} 
For two probability measures $\nu_1$ and $\nu_2$ in $\bR^d$, we define the $L^2(p)$ error of their scores as $\mathrm{SE}_{p}(\nu_1\|\nu_2):=\bE_{\bmx\sim p}[\|\nx\log\nu_1(\bmx)-\nx\log\nu_2(\bmx)\|^2]$ where $p$ also denotes a probability measure. 
Note that if we choose $p=\nu_1$, we have $\mathrm{SE}_{\nu_1}(\nu_1\|\nu_2)=F(\nu_1,\nu_2)$ where $F(\nu_1,\nu_2)$ denotes the Fisher divergence between $\nu_1$ and $\nu_2$. 
For any $\bma\in\bR^d$, we denote $\gamma_{\bma}(\bmx):=\exp(-\|\bmx-\bma\|^2/2)$. 
For simplify, we denote $\bE_{\bmx\sim\cN(\bma,I_d)}[\cdot]$ by $\bE_{\gamma_{\bma}}[\cdot]$. 
Thus the probability density of $\cN(\bma,I_d)$ is $p(\bmx)=\gamma_{\bma}(\bmx)/(\sqrt{2\pi})^d$. 
For the MoG 
$\pi^M=w_1\cN(\bma_1,I_d)+w_2\cN(\bma_2,I_d)$, the score is given by
\begin{equation}\label{eq: score of mog}
    \nx\log \pi^M(\bmx) = \frac{w_1\bma_1\gamma_{\bma_1}(\bmx)+w_2\bma_2\gamma_{\bma_2}(\bmx)}{w_1\gamma_{\bma_1}(\bmx)+w_2\gamma_{\bma_2}(\bmx)} - \bmx.
\end{equation}

Then we show our general results in Theorem \ref{thm: mog failure formal} where we state a lower bound of $\mathrm{KL}(\pi^{M}\|\tilde{\pi}^M)$ and an upper bound of $\mathrm{SE}_p(\pi^{M}\|\tilde{\pi}^M)$.
\begin{theorem}
    \label{thm: mog failure formal}
    Consider two MoGs in $\bR^d$: $\pi^M=w_1\mathcal{N}(\bma_1,I_d)+w_2\mathcal{N}(\bma_2,I_d)$, $\tilde{\pi}^M=\tilde{w}_1\mathcal{N}(\bma_1, I_d)+\tilde{w}_2\mathcal{N}(\bma_2, I_d)$ where $\bma_1,\bma_2\in\mathbb{R}^d$, $w_1,w_2,\tilde{w}_1,\tilde{w}_2>0$, $w_1+w_2=1$ and $\tilde{w}_1+\tilde{w}_2=1$. Then $\operatorname{KL}\left(\pi^M\|\tilde{\pi}^M\right)$ is lower bounded by
    \begin{equation}
        \label{lower KL bound: thm mog failure formal}
        \begin{aligned}
        \operatorname{KL}\left(\pi^M\|\tilde{\pi}^M\right)\ges 
        &\  w_1\left(\log w_1 - \log\left(\tilde{w}_1+\exp\left(-\frac{\left\|\bma_1-\bma_2\right\|^2}{4}\right)\right)\right) \\
        &\ +w_2\left(\log w_2 - \log\left(\tilde{w}_2+\exp\left(-\frac{\left\|\bma_1-\bma_2\right\|^2}{4}\right)\right)\right) \\
        & \ - \left(\log 4+d\right)\exp\left(\frac{d}{2}\log 2-\frac{\left\|\bma_1-\bma_2\right\|^2}{64}\right).
    \end{aligned}
    \end{equation}
    Let $p(\bmx)$ denotes any distribution that is absolutely continuous w.r.t. $\mu$, then $\mathrm{SE}_p(\pi^M\|\tilde{\pi}^M)$ is upper bounded by 
    \begin{equation}
        \label{upper SE bound: thm mog failure formal}
        \begin{aligned}
            \mathrm{SE}_p(\pi^M\|\tilde{\pi}^M)\les 
            & \ 2\exp\left(-\frac{\left\|\bma_1-\bma_2\right\|^2}{2}\right)\left[\frac{w_2^2}{w_1^2}+\frac{\tilde{w}_2^2}{\tilde{w}_1^2}+\frac{w_1^2}{w_2^2}+\frac{\tilde{w}_1^2}{\tilde{w}_2^2}\right]\left\|\bma_1-\bma_2\right\|^2 \\
            & \ + 8\left[\left\|\bma_1\right\|^2+\left\|\bma_2\right\|^2\right]\int_{\Omega_3}p(\bmx)\mathrm{~d}\bmx,
        \end{aligned}
    \end{equation}
    where $\Omega_1=\left\{\bmx\in\mathbb{R}^d:\left\|\bmx-\bma_1\right\|\les \frac{\left\|\bma_1-\bma_2\right\|}{4}\right\}$, $\Omega_2=\left\{\bmx\in\mathbb{R}^d:\left\|\bmx-\bma_2\right\|\les \frac{\left\|\bma_1-\bma_2\right\|}{4}\right\}$, and $\Omega_3=\Omega_1^c\bigcap\Omega_2^c$.
\end{theorem}

\begin{remark}
If we choose $p(\bmx)=\pi^M(\bmx)$ in Theorem \ref{thm: mog failure formal}, the Fisher divergence $F(\pi^M,\tilde{\pi}^M)$ is upper bounded by
\begin{equation}
\label{eq: fisher divergence upper bound}
\begin{aligned}
F(\pi^M,\tilde{\pi}^M)\les 
& \ 2\exp\left(-\frac{\left\|\bma_1-\bma_2\right\|^2}{2}\right)\left[\frac{w_2^2}{w_1^2}+\frac{\tilde{w}_2^2}{\tilde{w}_1^2}+\frac{w_1^2}{w_2^2}+\frac{\tilde{w}_1^2}{\tilde{w}_2^2}\right]\left\|\bma_1-\bma_2\right\|^2 \\
& \ + 8\left[\left\|\bma_1\right\|^2+\left\|\bma_2\right\|^2\right]\exp\left(\frac{d}{2}\log 2 - \frac{\left\|\bma_1-\bma_2\right\|^2}{64}\right),
\end{aligned}
\end{equation}
where we use the following inequality
\[
\begin{aligned}
\int_{\Omega_3}\pi^M(\bmx)\rmd\bmx 
= & w_1\int_{\Omega_3}\frac{1}{(\sqrt{2\pi})^d}\gamma_{\bma_1}(\bmx)\rmd\bmx + w_2\int_{\Omega_3}\frac{1}{(\sqrt{2\pi})^d}\gamma_{\bma_2}(\bmx)\rmd\bmx\\
\les & \ w_1\exp\left(\frac{d}{2}\log 2 - \frac{\left\|\bma_1-\bma_2\right\|^2}{64}\right)
+ w_2\exp\left(\frac{d}{2}\log 2 - \frac{\left\|\bma_1-\bma_2\right\|^2}{64}\right) \\
= & \exp\left(\frac{d}{2}\log 2 - \frac{\left\|\bma_1-\bma_2\right\|^2}{64}\right).
\end{aligned}
\]
Thus Example \ref{example: mog failure} holds naturally.
\end{remark}

\begin{proof}[Proof of Theorem \ref{thm: mog failure formal}]
We first prove \eqref{lower KL bound: thm mog failure formal}.
We can decompose $\operatorname{KL}\left(\pi^M\|\tilde{\pi}^M\right)$ as 
\begin{equation}
\label{eq: proof of mog failure: decompose KL}
\begin{aligned}
& \operatorname{KL}\left(\pi^M\|\tilde{\pi}^M\right) = \mathbb{E}_{\pi^M}\left[\log\left(\frac{\pi^M(\bmx)}{\tilde{\pi}^M(\bmx)}\right)\right]\\
=&\ w_1\mathbb{E}_{\gamma_{\bma_1}}\left[\log\left(\frac{w_1\gamma_{\bma_1}(\bmx)+w_2\gamma_{\bma_2}(\bmx)}{\tilde{w}_1\gamma_{\bma_1}(\bmx)+\tilde{w}_2\gamma_{\bma_2}(\bmx)}\right)\right] + w_2\mathbb{E}_{\gamma_{\bma_2}}\left[\log\left(\frac{w_1\gamma_{\bma_1}(\bmx)+w_2\gamma_{\bma_2}(\bmx)}{\tilde{w}_1\gamma_{\bma_1}(\bmx)+\tilde{w}_2\gamma_{\bma_2}(\bmx)}\right)\right].
\end{aligned}
\end{equation}
Note that
\begin{equation}
\label{eq: proof of mog failure: lower bound step 1}
\begin{aligned}
\mathbb{E}_{\gamma_{\bma_1}}\left[\log\left(\frac{w_1\gamma_{\bma_1}(\bmx)+w_2\gamma_{\bma_2}(\bmx)}{\tilde{w}_1\gamma_{\bma_1}(\bmx)+\tilde{w}_2\gamma_{\bma_2}(\bmx)}\right)\right] = & \mathbb{E}_{\gamma_{\bma_1}}\left[\log\left(\frac{w_1+w_2\gamma_{\bma_2}(\bmx)/\gamma_{\bma_1}(\bmx)}{\tilde{w}_1+\tilde{w}_2\gamma_{\bma_2}(\bmx)/\gamma_{\bma_1}(\bmx)}\right)\right]\\
\ges & \log w_1 - \mathbb{E}_{\gamma_{\bma_1}}\left[\log\left(\tilde{w}_1+\tilde{w}_2\frac{\gamma_{\bma_2}(\bmx)}{\gamma_{\bma_1}(\bmx)}\right)\right].
\end{aligned}
\end{equation}
Let $\widetilde{\Omega}_1=\{\bmx\in\mathbb{R}^d: \|\bmx-\bma_1\|\les \frac{\left\|\bma_1-\bma_2\right\|}{4}\}$, $\widetilde{\Omega}_2=\widetilde{\Omega}_1^c\bigcap\{\bmx\in\mathbb{R}^d:\tilde{w}_2\gamma_{\bma_2}(\bmx)/\gamma_{\bma_1}(\bmx)\les \tilde{w}_1\}$, and $\widetilde{\Omega}_3=(\widetilde{\Omega}_1\bigcup\widetilde{\Omega}_2)^c=\widetilde{\Omega}_1^c\bigcap\widetilde{\Omega}_2^c$. 
Then for any $\bmx\in\widetilde{\Omega}_1$, we have $\left\|\bmx-\bma_2\right\|\ges \left\|\bma_1-\bma_2\right\|-\left\|\bmx-\bma_1\right\|\ges \frac{3}{4}\left\|\bma_1-\bma_2\right\|$, thus $\gamma_{\bma_2}(\bmx)/\gamma_{\bma_1}(\bmx)=\exp\left(\frac{\left\|\bmx-\bma_1\right\|^2-\left\|\bmx-\bma_2\right\|^2}{2}\right)\les \exp\left(-\frac{\left\|\bma_1-\bma_2\right\|^2}{4}\right)$.
Then we have
\begin{equation}
\label{eq: proof of mog failure: Omega_1 KL}
\begin{aligned}
    & \int_{\widetilde{\Omega}_1}\frac{1}{\left(\sqrt{2\pi}\right)^d}\gamma_{\bma_1}\left(\bmx\right)\log\left(\tilde{w}_1+\tilde{w}_2\frac{\gamma_{\bma_2}(\bmx)}{\gamma_{\bma_1}(\bmx)}\right)\rmd\bmx\\
    \les & \ \log\left(\tilde{w}_1+\tilde{w}_2\exp\left(-\frac{\left\|\bma_1-\bma_2\right\|^2}{4}\right)\right)
    \les  \log\left(\tilde{w}_1+\exp\left(-\frac{\left\|\bma_1-\bma_2\right\|^2}{4}\right)\right).
\end{aligned}
\end{equation}
Note that 
\begin{equation}
\label{eq: proof of mog failure: useful note 1}
\begin{aligned}
    \int_{\widetilde{\Omega}_1^c}\frac{1}{\left(\sqrt{2\pi}\right)^d}\gamma_{\bma_1}\left(\bmx\right)\rmd\bmx\les & \int_{\widetilde{\Omega}_1^c}\exp\left(-\frac{\left\|\bma_1-\bma_2\right\|^2}{64}\right)\cdot\frac{1}{\left(\sqrt{2\pi}\right)^d}\exp\left(-\frac{\left\|\bmx-\bma_1\right\|^2}{4}\right)\rmd\bmx\\
    \les & \exp\left(\frac{d}{2}\log 2-\frac{\left\|\bma_1-\bma_2\right\|^2}{64}\right),
\end{aligned}
\end{equation}
and 
\begin{equation}
\label{eq: proof of mog failure: useful note 2}
\begin{aligned}
    & \int_{\widetilde{\Omega}_1^c}\frac{1}{\left(\sqrt{2\pi}\right)^d}\gamma_{\bma_1}\left(\bmx\right)\frac{\left\|\bmx-\bma_1\right\|^2}{2}\rmd\bmx\\
    \les & \int_{\widetilde{\Omega}_1^c}\exp\left(-\frac{\left\|\bma_1-\bma_2\right\|^2}{64}\right)\cdot\frac{1}{\left(\sqrt{2\pi}\right)^d}\exp\left(-\frac{\left\|\bmx-\bma_1\right\|^2}{4}\right)\frac{\left\|\bmx-\bma_1\right\|^2}{2}\rmd\bmx\\
    \les & \exp\left(\frac{d}{2}\log 2-\frac{\left\|\bma_1-\bma_2\right\|^2}{64}\right)\cdot \mathbb{E}_{\bmx\sim\mathcal{N}(\bma_1,2I_d)}\left[\frac{\left\|\bmx-\bma_1\right\|^2}{2}\right]\\
    =& \exp\left(\log d+\frac{d}{2}\log 2-\frac{\left\|\bma_1-\bma_2\right\|^2}{64}\right).
\end{aligned}
\end{equation}
For every $\bmx\in\widetilde{\Omega}_2$, we have $\log\left(\tilde{w}_1+\tilde{w}_2\gamma_{\bma_2}(\bmx)/\gamma_{\bma_1}(\bmx)\right)\les \log 2$. Using (\ref{eq: proof of mog failure: useful note 1}) and (\ref{eq: proof of mog failure: useful note 2}), we obtain
\begin{equation}
\label{eq: proof of mog failure: Omega_2 KL }
\begin{aligned}
& \int_{\widetilde{\Omega}_2}\frac{1}{\left(\sqrt{2\pi}\right)^d}\gamma_{\bma_1}\left(\bmx\right)\log\left(\tilde{w}_1+\tilde{w}_2\frac{\gamma_{\bma_2}(\bmx)}{\gamma_{\bma_1}(\bmx)}\right)\rmd\bmx\les \log 2\cdot \exp\left(\frac{d}{2}\log 2-\frac{\left\|\bma_1-\bma_2\right\|^2}{64}\right).
\end{aligned}
\end{equation}
Similarly, for any $\bmx\in\widetilde{\Omega}_3$, we have $\log\left(\tilde{w}_1+\tilde{w}_2\gamma_{\bma_2}(\bmx)/\gamma_{\bma_1}(\bmx)\right)\les\log 2+\frac{\left\|\bmx-\bma_1\right\|^2}{2}$. Thus,
\begin{equation}
\label{eq: proof of mog failure: Omega_3 KL }
\begin{aligned}
\int_{\widetilde{\Omega}_3}\frac{1}{\left(\sqrt{2\pi}\right)^d}\gamma_{\bma_1}\left(\bmx\right)\log\left(\tilde{w}_1+\tilde{w}_2\frac{\gamma_{\bma_2}(\bmx)}{\gamma_{\bma_1}(\bmx)}\right)\rmd\bmx \les \left(\log 2+d\right)\exp\left(\frac{d}{2}\log 2-\frac{\left\|\bma_1-\bma_2\right\|^2}{64}\right).
\end{aligned}
\end{equation}
Putting (\ref{eq: proof of mog failure: Omega_1 KL}), (\ref{eq: proof of mog failure: Omega_2 KL }), and (\ref{eq: proof of mog failure: Omega_3 KL }) together, $\mathbb{E}_{\gamma_{\bma_1}}\left[\log\left(\tilde{w}_1+\tilde{w}_2\gamma_{\bma_2}(\bmx)/\gamma_{\bma_1}(\bmx)\right)\right]$ is upper bounded by 
\begin{equation}
\label{eq: proof of mog failure: KL upper bounded 1}
\begin{aligned}
&\mathbb{E}_{\gamma_{\bma_1}}\left[\log\left(\tilde{w}_1+\tilde{w}_2\frac{\gamma_{\bma_2}(\bmx)}{\gamma_{\bma_1}(\bmx)}\right)\right]\\
= &\left(\int_{\widetilde{\Omega}_1}+\int_{\widetilde{\Omega}_2}+\int_{\widetilde{\Omega}_3}\right)\frac{1}{\left(\sqrt{2\pi}\right)^d}\gamma_{\bma_1}\left(\bmx\right)\log\left(\tilde{w}_1+\tilde{w}_2\frac{\gamma_{\bma_2}(\bmx)}{\gamma_{\bma_1}(\bmx)}\right)\rmd\bmx\\
\les & \log\left(\tilde{w}_1+\exp\left(-\frac{\left\|\bma_1-\bma_2\right\|^2}{4}\right)\right) + \left(\log 4+d\right)\exp\left(\frac{d}{2}\log 2-\frac{\left\|\bma_1-\bma_2\right\|^2}{64}\right).
\end{aligned}
\end{equation}
Plugging (\ref{eq: proof of mog failure: KL upper bounded 1}) into (\ref{eq: proof of mog failure: lower bound step 1}), we have
\begin{equation}
\label{eq: proof of mog failure: KL upper bound 2}
\begin{aligned} & \ w_1\mathbb{E}_{\gamma_{\bma_1}}\left[\log\left(\frac{w_1\gamma_{\bma_1}(\bmx)+w_2\gamma_{\bma_2}(\bmx)}{\tilde{w}_1\gamma_{\bma_1}(\bmx)+\tilde{w}_2\gamma_{\bma_2}(\bmx)}\right)\right]\\
\ges & \ w_1\left[\log w_1 - \log\left(\tilde{w}_1+\exp\left(-\frac{\left\|\bma_1-\bma_2\right\|^2}{4}\right)\right)\right] \\
& \quad - w_1\left(\log 4+d\right)\exp\left(\frac{d}{2}\log 2-\frac{\left\|\bma_1-\bma_2\right\|^2}{64}\right).
\end{aligned}
\end{equation}
Similarly, we have \begin{equation}
\label{eq: proof of mog failure: KL upper bound 3}
\begin{aligned}& \ w_2\mathbb{E}_{\gamma_{\bma_2}}\left[\log\left(\frac{w_1\gamma_{\bma_1}(\bmx)+w_2\gamma_{\bma_2}(\bmx)}{\tilde{w}_1\gamma_{\bma_1}(\bmx)+\tilde{w}_2\gamma_{\bma_2}(\bmx)}\right)\right]\\
\ges & \ w_2\left(\log w_2 - \log\left(\tilde{w}_2+\exp\left(-\frac{\left\|\bma_1-\bma_2\right\|^2}{4}\right)\right)\right) \\
&\quad - w_2\left(\log 4+d\right)\exp\left(\frac{d}{2}\log 2-\frac{\left\|\bma_1-\bma_2\right\|^2}{64}\right).
\end{aligned}
\end{equation}
Plugging (\ref{eq: proof of mog failure: KL upper bound 2}) and (\ref{eq: proof of mog failure: KL upper bound 3}) into (\ref{eq: proof of mog failure: decompose KL}), we obtain the lower bound (\ref{lower KL bound: thm mog failure formal}) in Theorem \ref{thm: mog failure formal}.
Then we prove \eqref{upper SE bound: thm mog failure formal}.
Using (\ref{eq: score of mog}), we obtain 
\begin{equation}
\label{eq: proof of mog failure: score error rewriten note 1}
\begin{aligned}
& \nx\log\pi^M(\bmx)-\nx\log\tilde{\pi}^M(\bmx)\\
= & \frac{w_1\bma_1\gamma_{\bma_1}(\bmx)+w_2\bma_2\gamma_{\bma_2}(\bmx)}{w_1\gamma_{\bma_1}(\bmx)+w_2\gamma_{\bma_2}(\bmx)} - \frac{\tilde{w}_1\bma_1\gamma_{\bma_1}(\bmx)+\tilde{w}_2\bma_2\gamma_{\bma_2}(\bmx)}{\tilde{w}_1\gamma_{\bma_1}(\bmx)+\tilde{w}_2\gamma_{\bma_2}(\bmx)}\\
= & \frac{w_1\bma_1+w_2\bma_2\gamma_{\bma_2}(\bmx)/\gamma_{\bma_1}(\bmx)}{w_1+w_2\gamma_{\bma_2}(\bmx)/\gamma_{\bma_1}(\bmx)} - \frac{\tilde{w}_1\bma_1+\tilde{w}_2\bma_2\gamma_{\bma_2}(\bmx)/\gamma_{\bma_1}(\bmx)}{\tilde{w}_1+\tilde{w}_2\gamma_{\bma_2}(\bmx)/\gamma_{\bma_1}(\bmx)}\\
= & \frac{w_1\bma_1\gamma_{\bma_1}(\bmx)/\gamma_{\bma_2}(\bmx)+w_2\bma_2}{w_1\gamma_{\bma_1}(\bmx)/\gamma_{\bma_2}(\bmx)+w_2} - \frac{\tilde{w}_1\bma_1\gamma_{\bma_1}(\bmx)/\gamma_{\bma_2}(\bmx)+\tilde{w}_2\bma_2}{\tilde{w}_1\gamma_{\bma_1}(\bmx)/\gamma_{\bma_2}(\bmx)+\tilde{w}_2}.
\end{aligned}
\end{equation}
Recall that $\Omega_1=\{\bmx\in\mathbb{R}^d:\left\|\bmx-\bma_1\right\|\les \frac{\left\|\bma_1-\bma_2\right\|}{4}\}$, $\Omega_2=\{\bmx\in\mathbb{R}^d:\left\|\bmx-\bma_2\right\|\les \frac{\left\|\bma_1-\bma_2\right\|}{4}\}$ and $\Omega_3=\Omega_1^c\bigcap\Omega_2^c$. 
For any $\bmx\in\Omega_1$, we can rewrite (\ref{eq: proof of mog failure: score error rewriten note 1}) as \begin{equation}
\label{eq: proof of mog failure: score error rewritten part 1}
\begin{aligned}
& \nx\log\pi^M(\bmx)-\nx\log\tilde{\pi}^M(\bmx)\\
=  & \bma_1 + \frac{w_2(\bma_2-\bma_1)\gamma_{\bma_2}(\bmx)/\gamma_{\bma_1}(\bmx)}{w_1+w_2\gamma_{\bma_2}(\bmx)/\gamma_{\bma_1}(\bmx)}-\left\{\bma_1+\frac{\tilde{w}_2(\bma_2-\bma_1)\gamma_{\bma_2}(\bmx)/\gamma_{\bma_1}(\bmx)}{\tilde{w}_1+\tilde{w}_2\gamma_{\bma_2}(\bmx)/\gamma_{\bma_1}(\bmx)}\right\}\\
= & \frac{w_2(\bma_2-\bma_1)\gamma_{\bma_2}(\bmx)/\gamma_{\bma_1}(\bmx)}{w_1+w_2\gamma_{\bma_2}(\bmx)/\gamma_{\bma_1}(\bmx)}-\frac{\tilde{w}_2(\bma_2-\bma_1)\gamma_{\bma_2}(\bmx)/\gamma_{\bma_1}(\bmx)}{\tilde{w}_1+\tilde{w}_2\gamma_{\bma_2}(\bmx)/\gamma_{\bma_1}(\bmx)}.
\end{aligned}
\end{equation}
Note that $\vert\gamma_{\bma_2}(\bmx)/\gamma_{\bma_1}(\bmx)\vert^2=\exp(\|\bmx-\bma_1\|^2-\|\bmx-\bma_2\|^2)\les \exp(-\|\bma_1-\bma_2\|^2/2)$ for every $\bmx\in\Omega_1$. Then use (\ref{eq: proof of mog failure: score error rewritten part 1}), we obtain 
\begin{equation}
\label{eq: proof of mog failure: part 1 score error}
\begin{aligned}
& \int_{\Omega_1}\left\|\nx\log\pi^M(\bmx)-\nx\log\tilde{\pi}^M(\bmx)\right\|^2p(\bmx)\rmd\bmx\\
= & \int_{\Omega_1}\left\|\frac{w_2(\bma_2-\bma_1)\gamma_{\bma_2}(\bmx)/\gamma_{\bma_1}(\bmx)}{w_1+w_2\gamma_{\bma_2}(\bmx)/\gamma_{\bma_1}(\bmx)}-\frac{\tilde{w}_2(\bma_2-\bma_1)\gamma_{\bma_2}(\bmx)/\gamma_{\bma_1}(\bmx)}{\tilde{w}_1+\tilde{w}_2\gamma_{\bma_2}(\bmx)/\gamma_{\bma_1}(\bmx)}\right\|^2p(\bmx)\rmd\bmx\\
\les & \ 2\int_{\Omega_1}\left\|\frac{w_2(\bma_2-\bma_1)}{w_1}\right\|^2 \Big\vert\frac{\gamma_{\bma_2}(\bmx)}{\gamma_{\bma_1}(\bmx)}\Big\vert^2p(\bmx)\rmd\bmx
+2\int_{\Omega_1}\left\|\frac{\tilde{w}_2(\bma_2-\bma_1)}{\tilde{w}_1}\right\|^2 \Big\vert
\frac{\gamma_{\bma_2}(\bmx)}{\gamma_{\bma_1}(\bmx)}\Big\vert^2p(\bmx)\rmd\bmx\\
\les & \ 2\exp\left(-\frac{\left\|\bma_1-\bma_2\right\|^2}{2}\right)\left[\frac{w_2^2}{w_1^2}+\frac{\tilde{w}_2^2}{\tilde{w}_1^2}\right]\left\|\bma_2-\bma_1\right\|^2.
\end{aligned}
\end{equation}
Similarly, we obtain 
\begin{equation}
\label{eq: proof of mog failure: part 2 score error}
\begin{aligned}
& \int_{\Omega_2}\left\|\nx\log\pi^M(\bmx)-\nx\log\tilde{\pi}^M(\bmx)\right\|^2p(\bmx)\rmd\bmx \\
\les & \ 2\exp\left(-\frac{\left\|\bma_1-\bma_2\right\|^2}{2}\right)\left[\frac{w_1^2}{w_2^2}+\frac{\tilde{w}_1^2}{\tilde{w}_2^2}\right]\left\|\bma_1-\bma_2\right\|^2.
\end{aligned}
\end{equation}
Using (\ref{eq: proof of mog failure: score error rewriten note 1}), we obtain that 
\begin{equation}
\label{eq: proof of mog failure: part 3 score error}
\begin{aligned}
& \int_{\Omega_3}\left\|\nx\log\pi^M(\bmx)-\nx\log\tilde{\pi}^M(\bmx)\right\|^2p(\bmx)\rmd\bmx\\
= & \int_{\Omega_3}\left\|\frac{w_1\bma_1+w_2\bma_2\gamma_{\bma_2}(\bmx)/\gamma_{\bma_1}(\bmx)}{w_1+w_2\gamma_{\bma_2}(\bmx)/\gamma_{\bma_1}(\bmx)} - \frac{\tilde{w}_1\bma_1+\tilde{w}_2\bma_2\gamma_{\bma_2}(\bmx)/\gamma_{\bma_1}(\bmx)}{\tilde{w}_1+\tilde{w}_2\gamma_{\bma_2}(\bmx)/\gamma_{\bma_1}(\bmx)}\right\|^2p(\bmx)\rmd\bmx\\
\les & \ 4\int_{\Omega_3}\left\|\frac{w_1\bma_1}{w_1+w_2\gamma_{\bma_2}(\bmx)/\gamma_{\bma_1}(\bmx)}\right\|^2p(\bmx)\rmd\bmx+4\int_{\Omega_3}\left\|\frac{\tilde{w}_1\bma_1}{\tilde{w}_1+\tilde{w}_2\gamma_{\bma_2}(\bmx)/\gamma_{\bma_1}(\bmx)}\right\|^2p(\bmx)\rmd\bmx\\
& \quad + \ 4\int_{\Omega_3}\left\|\frac{w_2\bma_2\gamma_{\bma_2}(\bmx)/\gamma_{\bma_1}(\bmx)}{w_1+w_2\gamma_{\bma_2}(\bmx)/\gamma_{\bma_1}(\bmx)}\right\|^2p(\bmx)\rmd\bmx+4\int_{\Omega_3}\left\|\frac{\tilde{w}_2\bma_2\gamma_{\bma_2}(\bmx)/\gamma_{\bma_1}(\bmx)}{\tilde{w}_1+\tilde{w}_2\gamma_{\bma_2}(\bmx)/\gamma_{\bma_1}(\bmx)}\right\|^2p(\bmx)\rmd\bmx\\
\les & \ 8\left[\left\|\bma_1\right\|^2+\left\|\bma_2\right\|^2\right]\int_{\Omega_3}p(\bmx)\rmd\bmx.
\end{aligned}
\end{equation}
Note that we have the following decomposition\begin{equation}
\label{eq: proof of mog failure: decompose score error}
\begin{aligned}
\mathrm{SE}_p(\pi^M\|\tilde{\pi}^M)
=& \int_{\mathbb{R}^d}\left\|\nx\log \pi^M(\bmx)-\nx\log\tilde{\pi}^M(\bmx)\right\|^2p(\bmx)\rmd\bmx\\
=& \left(\int_{\Omega_1}+\int_{\Omega_2}+\int_{\Omega_3}\right)\left\|\nx\log \pi^M(\bmx)-\nx\log\tilde{\pi}^M(\bmx)\right\|^2p(\bmx)\rmd\bmx.
\end{aligned}
\end{equation}
Plug (\ref{eq: proof of mog failure: part 1 score error}), (\ref{eq: proof of mog failure: part 2 score error}), and (\ref{eq: proof of mog failure: part 3 score error}) into (\ref{eq: proof of mog failure: decompose score error}), we obtain the upper bound (\ref{upper SE bound: thm mog failure formal}) in Theorem \ref{thm: mog failure formal}.
\end{proof}

\subsection{Proof of Theorem \ref{thm: PINN analysis}}
\label{appendix: Theorem PINN analysis proof}
First, we present the divergence theorem and Green's first identity, which is very useful in our proof.
Then we state the Gr\"{o}nwall's inequality used in our proof.
Finally, we state and prove Theorem \ref{thm: extended PINN analysis} which includes Theorem \ref{thm: PINN analysis} and sharper bounds when \eqref{eq: thm PINN analysis poincare ineq} holds.

\begin{lemma}[divergence theorem]\label{lemma: divergence theorem}
    Let $\mathbf{F}(\cdot):\Omega\to \mathbb{R}^d$, then
    $\int_{\Omega}\nabla\cdot \mathbf{F}(\bmx)\rmd\bmx=\int_{\partial\Omega}\mathbf{F}\cdot \bm{n}\mathrm{~d}\bm{S}$.
\end{lemma}

\begin{lemma}[Green's first identity]\label{lemma: green's first identity}
Let $v(\cdot), u(\cdot): \Omega\to \mathbb{R}$, then it holds that
\[
\int_{\Omega}\nabla_{\bmx} v\cdot\nabla_{\bmx} u\rmd\bmx+\int_{\Omega}v\Delta u\rmd\bmx=\int_{\partial\Omega}v\frac{\partial u}{\partial \bm{n}}\mathrm{~d}\bm{S}.
\]
\end{lemma}

\begin{lemma}[Gr\"{o}nwall's inequality]\label{Gronwall}
Let $f(\cdot),\alpha(\cdot), \beta(\cdot): [0, T]\to\mathbb{R}$, and suppose that $\forall\ 0\les t\les T$,\[
f^{\prime}(t)\les \alpha(t) + \beta(t)f(t).
\]
Then we have $\forall\ 0\les t\les T$,\[
f(t)\les e^{\int_0^t\beta(s)\rmd s}f(0) + \int_0^t e^{\int_s^t\beta(r)\rmd r}\alpha(s)\rmd s.
\]
\end{lemma}

\begin{proof}[Proof of Lemma \ref{Gronwall}]
Consider $g(t)=e^{-\int_0^t\beta(s)\rmd s}f(t),\forall\ 0\les t\les T$. Then we have $\forall\ 0\les t\les T$,
\begin{equation}\label{eq: function}
\begin{aligned}
g^{\prime}(t)=& \ e^{-\int_0^t\beta(s)\rmd s}f^{\prime}(t) - \beta(t)e^{-\int_0^t\beta(s)\rmd s}f(t)\\
= & \ e^{-\int_0^t\beta(s)\rmd s}\left(f^{\prime}(t)-\beta(t)f(t)\right)\\
\les & \ e^{-\int_0^t\beta(s)\rmd s}\alpha(t).
\end{aligned}
\end{equation}
Integrating (\ref{eq: function}), we obtain 
\begin{equation}\label{eq: integrate function}
e^{-\int_0^t\beta(s)\rmd s}f(t)\les f(0) + \int_0^t e^{-\int_0^s\beta(r)\rmd r}\alpha(s)\rmd s.
\end{equation}
Hence, we complete our proof.
\end{proof}

\begin{theorem}
\label{thm: extended PINN analysis}
Suppose that Assumption \ref{PINN assumption: boundary identity}, \ref{PINN assumption: diffusion bound}, and \ref{PINN assumption: continuous assumption} hold. 
We further assume that $u_\theta(\bmx,0)=u^*_0(\bmx)$ for any $\bmx\in\Omega$. 
Then for any positive constant $\vareps>0$, the following holds for any $0\les t\les T$,
\begin{equation}\label{eq: thm PINN analysis e 1}
    \|e_t(\cdot)\|^2_{L^2(\Omega;\nu_t)}\les \vareps L_{\emph{\textrm{PINN}}}(t;C_1(\vareps)).
\end{equation}
Moreover, for any $0\les t\les T$,
\begin{equation}\label{eq: thm PINN analysis nabla e 1}
m_1\|\nx e_t(\cdot)\|^2_{L^2(\Omega;\nu_t)}\les
\vareps\|r_t(\cdot)\|^2_{L^2(\Omega;\nu_t)}
+C_3(\vareps) L_{\emph{\textrm{PINN}}}(t;C_1(\vareps)) 
+C_2\sqrt{\vareps  L_{\emph{\textrm{PINN}}}(t;C_1(\vareps))}.
\end{equation}
In addition, if there exists constant $\mathcal{C}_{\nu}(\Omega)>0$ such that the following holds for any $0\les t\les T$,
\begin{equation}\label{eq: thm PINN analysis poincare ineq}
    \|\nx e_t(\cdot)\|^2_{L^2(\Omega;\nu_t)}\ges \mathcal{C}^2_{\nu}(\Omega) \|e_t(\cdot)\|^2_{L^2(\Omega;\nu_t)}.
\end{equation}
Then for any positive constant $\vareps>0$, the following holds for any $0\les t\les T$,
\begin{equation}\label{eq: thm PINN analysis ineq 3}
\|e_t(\cdot)\|^2_{L^2(\Omega;\nu_t)}\les  \vareps L_{\emph{\textrm{PINN}}}(t;C_4(\vareps)).
\end{equation}
Moreover, for any $0\les t\les T$,
\begin{equation}\label{eq: thm PINN analysis ineq 4}
m_1\|\nx e_t(\cdot)\|^2_{L^2(\Omega;\nu_t)}\les 
\vareps\|r_t(\cdot)\|^2_{L^2(\Omega;\nu_t)}
+C_3(\vareps) L_{\emph{\textrm{PINN}}}(t;C_4(\vareps))
+C_2\sqrt{\vareps L_{\emph{\textrm{PINN}}}(t;C_4(\vareps))}.
\end{equation}
where $C_2:=2\sqrt{2}(\hB_0^2+B_0^{*2})^{1/2}$, $C_3(\vareps):=\vareps(C_1(\vareps)+\Bv_0)$, $C_4(\vareps):=C_1(\vareps)-m_1\mathcal{C}^2_{\nu}(\Omega)$ and
\begin{equation*}
C_1(\vareps):=
\frac{1}{\vareps}+\frac{M_1}{4}(\Bv_1+2\Bu_1+2\hB_1)+c_1(\Bv_1+\Bv_2)-\frac{c_2}{2}(\Bu_2+\hB_2),
\end{equation*}
where 
\begin{equation*}
c_1:=\begin{cases}
M_1, \quad \text{if }\Bv_1+\Bv_2\ges 0.\\
m_1, \quad \text{if } \Bv_1+\Bv_2 < 0.
\end{cases},\quad
c_2:=\begin{cases}
m_1, \quad \text{if }\Bu_2+\hB_2\ges 0.\\
M_1, \quad \text{if }\Bu_2+\hB_2 < 0.
\end{cases}.
\end{equation*}
\end{theorem}

\begin{proof}[Proof of Theorem \ref{thm: extended PINN analysis}]
We first prove \eqref{eq: thm PINN analysis e 1} and \eqref{eq: thm PINN analysis ineq 3}. 
Note that $u_t^*(\bmx)$ satisfies
\begin{equation}\label{eq: pde for u^*}
\pt u^*_t(\bmx) 
+ \nx u^*_t(\bmx)\cdot\bmf(\bmx, t) 
+ \diver \bmf(\bmx, t) 
- \frac{1}{2}g^2(t)\Lap u_t^*(\bmx)
- \frac{1}{2}g^2(t)\left\|\nx u_t^*(\bmx)\right\|^2
= 0.
\end{equation}
And $ u_\theta(\bmx, t)$ satisfies
\begin{equation}\label{eq: pde for hu}
\pt  u_\theta(\bmx, t) 
+ \nx  u_\theta(\bmx, t)\cdot\bmf(\bmx, t) 
+ \diver \bmf(\bmx, t) 
- \frac{1}{2}g^2(t)\Lap  u_\theta(\bmx, t)
- \frac{1}{2}g^2(t)\left\|\nx  u_\theta(\bmx, t)\right\|^2
= r_t(\bmx).
\end{equation}
Subtracting \eqref{eq: pde for u^*} for $u^*$ from \eqref{eq: pde for hu} for $u_\theta$, we have
\begin{equation}\label{eq: PINN analysis proof e_t's PDE}
\partial_t e_t(\bmx)
+\nabla_{\bmx} e_t(\bmx)\cdot \bmf(\bmx, t)
-\frac{1}{2}g^2(t)\left(\left\|\nabla_{\bmx} u_\theta(\bmx,t)\right\|^2 - \left\|\nabla_{\bmx} u^*_t(\bmx)\right\|^2\right)
-\frac{1}{2}g^2(t)\Delta e_t(\bmx)=r_t(\bmx).
\end{equation}
Note that $\frac{1}{2}\partial_t e_t^2(\bmx)=e_t(\bmx)\partial_t e_t(\bmx)$ and $\frac{1}{2}\nabla_{\bmx} e_t^2(\bmx)=e_t(\bmx)\nabla_{\bmx} e_t(\bmx)$, then we obtain
\begin{equation}\label{eq: PINN analysis proof partial_t e_t}
\begin{aligned}
\frac{1}{2}\partial_t e_t^2(\bmx) 
= & \ \frac{1}{2}g^2(t)e_t(\bmx)\left(\left\|\nabla_{\bmx}   u_\theta(\bmx, t)\right\|^2 - \left\|\nabla_{\bmx} u^*_t(\bmx)\right\|^2\right) 
+ \frac{1}{2}g^2(t)e_t(\bmx)\Delta e_t(\bmx) \\
& \quad + e_t(\bmx)r_t(\bmx) 
- e_t(\bmx)\nabla_{\bmx} e_t(\bmx)\cdot \bmf(\bmx, t)\\
= & \ \frac{1}{2}g^2(t)e_t(\bmx)\nabla_{\bmx} e_t(\bmx)\cdot\left(\nabla_{\bmx}   u_\theta(\bmx, t) + \nabla_{\bmx} u^*_t(\bmx)\right) 
+ \frac{1}{2}g^2(t)e_t(\bmx)\Delta e_t(\bmx) \\
&\quad + e_t(\bmx)r_t(\bmx) 
- e_t(\bmx)\nabla_{\bmx} e_t(\bmx)\cdot \bmf(\bmx, t)\\
= & \ \frac{1}{4}g^2(t)\nabla_{\bmx} e_t^2(\bmx)\cdot\left(\nabla_{\bmx}   u_\theta(\bmx, t) + \nabla_{\bmx} u^*_t(\bmx)\right) 
+ \frac{1}{2}g^2(t)e_t(\bmx)\Delta e_t(\bmx)\\
&\quad + e_t(\bmx)r_t(\bmx)
- \frac{1}{2}\nabla_{\bmx} e_t^2(\bmx)\cdot f(\bmx, t).
\end{aligned}
\end{equation}
Note that $\pt(\nu_t(\bmx)e_t^2(\bmx))=e_t^2(\bmx)\pt\nu_t(\bmx)+\nu_t(\bmx)\pt e_t^2(\bmx)$, then we have
\begin{equation}\label{eq: PINN analysis proof pde of ev}
\begin{aligned}
\partial_t(\nu_t(\bmx)e^2_t(\bmx)) 
= &\ \frac{1}{2}g^2(t)\nu_t(\bmx)\nabla_{\bmx} e_t^2(\bmx)\cdot\left(\nabla_{\bmx}   u_\theta(\bmx, t) + \nabla_{\bmx} u^*_t(\bmx)\right) \\
& \quad + g^2(t)\nu_t(\bmx)e_t(\bmx)\Delta e_t(\bmx) \\
& \quad + 2\nu_t(\bmx)e_t(\bmx)r_t(\bmx)
- \nu_t(\bmx)\nabla_{\bmx} e_t^2(\bmx)\cdot \bmf(\bmx, t) \\
& \quad + \frac{1}{2}g^2(t)e_t^2(\bmx)\Delta \nu_t(\bmx) 
- e_t^2(\bmx)\nabla\cdot\left[\bmf(\bmx, t)\nu_t(\bmx)\right].
\end{aligned}
\end{equation}
We integrate (\ref{eq: PINN analysis proof pde of ev}) to get an equation for $\|e_t(\cdot)\|^2_{L^2(\Omega;\nu_t)}$ given by 
\begin{equation}\label{eq: PINN analysis proof ev integrate}
\begin{aligned}
\pt\left\|e_t(\cdot)\right\|^2_{L^2(\Omega;\nu_t)}
= &\ \frac{1}{2}g^2(t)\int_{\Omega}\nu_t(\bmx)\nabla_{\bmx} e_t^2(\bmx)\cdot\left(\nabla_{\bmx}   u_\theta(\bmx, t) + \nabla_{\bmx} u^*_t(\bmx)\right) \rmd\bmx \\
& \quad + \ g^2(t)\int_{\Omega}\nu_t(\bmx)e_t(\bmx)\Delta e_t(\bmx) \rmd\bmx\\
& \quad + 2\int_{\Omega}\nu_t(\bmx)e_t(\bmx)r_t(\bmx)\rmd\bmx
- \int_{\Omega}\nu_t(\bmx)\nabla_{\bmx} e_t^2(\bmx)\cdot\bmf(\bmx, t) \rmd\bmx\\
& \quad + \frac{1}{2}g^2(t)\int_{\Omega}e_t^2(\bmx)\Delta\nu_t(\bmx)\rmd\bmx - \int_{\Omega}e_t^2(\bmx)\nabla\cdot\left[\bmf(\bmx, t)\nu_t(\bmx)\right]\rmd\bmx.
\end{aligned}
\end{equation}
Note that 
\[
\diver\left[\nu_t(\bmx)e_t^2(\bmx)\bmf(\bmx, t)\right]
=\nu_t(\bmx)\nx e_t^2(\bmx)\cdot\bmf(\bmx, t)
+e_t^2(\bmx)\diver\left[\nu_t(\bmx)\bmf(\bmx, t)\right].
\]
Then using Lemma \ref{lemma: divergence theorem} and $e_t(\bmx)=0$ for any $(\bmx, t)\in\partial\Omega\times [0, T]$, we have
\begin{equation}\label{eq: identity about f}
    \int_{\Omega}\nu_t(\bmx)\nx e_t^2(\bmx)\cdot\bmf(\bmx, t)\rmd\bmx
    +\int_{\Omega}e_t^2(\bmx)\diver\left[\nu_t(\bmx)\bmf(\bmx, t)\right]\rmd\bmx
    =0.
\end{equation}
Similarly, we have
\begin{equation}\label{eq: identity about u}
\begin{aligned} 
& \int_{\Omega}\nu_t(\bmx)\nx e_t^2(\bmx)\cdot\left(\nx u_\theta(\bmx, t)+\nx u^*_t(\bmx)\right)\rmd\bmx\\
=
& -\int_{\Omega}\nu_t(\bmx)e_t^2(\bmx)\left(\Lap u_\theta(\bmx, t)+\Lap u^*_t(\bmx)\right)\rmd\bmx\\
& -\int_{\Omega}e_t^2(\bmx)\nx\nu_t(\bmx)\cdot \left(\nx u_\theta(\bmx, t)+\nx u^*_t(\bmx)\right)\rmd\bmx,
\end{aligned}
\end{equation}
and 
\begin{equation}\label{eq: idenetity about e nabla e}
\int_{\Omega}\nu_t(\bmx)e_t(\bmx)\Lap e_t(\bmx)\rmd\bmx 
=  
-\frac{1}{2}\int_{\Omega}\nx\nu_t(\bmx)\cdot \nx e_t^2(\bmx)\rmd\bmx 
-\int_{\Omega}\nu_t(\bmx)\left\|\nx e_t(\bmx)\right\|^2\rmd\bmx.
\end{equation}
Plugging (\ref{eq: identity about f}), (\ref{eq: identity about u}), and (\ref{eq: idenetity about e nabla e}) into (\ref{eq: PINN analysis proof ev integrate}), and using Lemma \ref{lemma: green's first identity}, we have
\begin{equation}\label{eq: integrate by part 1}
\begin{aligned}
\pt \|e_t(\cdot)\|^2_{L^2(\Omega;\nu_t)}
= & -\frac{1}{2}g^2(t)\int_{\Omega}(\Lap u_\theta(\bmx, t)+\Lap u^*_t(\bmx))e_t^2(\bmx)\nu_t(\bmx)\rmd\bmx \\
& \ - \frac{1}{2}g^2(t)\int_{\Omega}\nx \nu_t(\bmx)\cdot\left(\nx  u_\theta(\bmx, t)+\nx u_t^*(\bmx)\right)e_t^2(\bmx)\rmd\bmx \\
& \ - g^2(t)\int_{\Omega}\nu_t(\bmx)\|\nx e_t(\bmx)\|^2\rmd\bmx + 2\int_{\Omega}\nu_t(\bmx)e_t(\bmx)r_t(\bmx)\rmd\bmx\\
& \ + g^2(t)\int_{\Omega}e_t^2(\bmx)\Lap \nu_t(\bmx)\rmd\bmx.
\end{aligned}
\end{equation}
Using $\nx\nu_t(\bmx)=\nu_t(\bmx)\nx\log \nu_t(\bmx)$ and $\Lap \nu_t(\bmx)=(\Lap\log \nu_t(\bmx)+\|\nx \log \nu_t(\bmx)\|^2)\nu_t(\bmx)$,
\begin{equation}\label{eq: integrate by part 2}
\begin{aligned}
\pt \|e_t(\cdot)\|^2_{L^2(\Omega;\nu_t)}
= & -\frac{1}{2}g^2(t)\int_{\Omega}(\Lap u_\theta(\bmx, t)+\Lap u^*_t(\bmx))e_t^2(\bmx)\nu_t(\bmx)\rmd\bmx \\
& \ - \frac{1}{2}g^2(t)\int_{\Omega}\nx \log \nu_t(\bmx)\cdot\left(\nx  u_\theta(\bmx, t)+\nx u_t^*(\bmx)\right)e_t^2(\bmx)\nu_t(\bmx)\rmd\bmx \\
& \ - g^2(t)\int_{\Omega}\nu_t(\bmx)\|\nx e_t(\bmx)\|^2\rmd\bmx + 2\int_{\Omega}\nu_t(\bmx)e_t(\bmx)r_t(\bmx)\rmd\bmx\\
& \ + g^2(t)\int_{\Omega}e_t^2(\bmx)\nu_t(\bmx)(\Lap\log \nu_t(\bmx)+\|\nx \log \nu_t(\bmx)\|^2)\rmd\bmx.
\end{aligned}
\end{equation}
By Assumption \ref{PINN assumption: diffusion bound} and \ref{PINN assumption: continuous assumption}, then we have
\begin{equation}\label{eq: PINN analysis proof ineq 1}
\begin{aligned}
\pt \|e_t(\cdot)\|_{L^2(\Omega;\nu_t)}^2 \les & \
\vareps \|r_t(\cdot)\|^2_{L^2(\Omega;\nu_t)} + C_1(\vareps)\|e_t(\cdot)\|_{L^2(\Omega;\nu_t)}^2-m_1\|\nx e_t(\cdot)\|^2_{L^2(\Omega;\nu_t)}\\
\les & \ \vareps \|r_t(\cdot)\|^2_{L^2(\Omega;\nu_t)} + C_1(\vareps)\|e_t(\cdot)\|_{L^2(\Omega;\nu_t)}^2,
\end{aligned}
\end{equation}
which follows from applying Young's inequality and holds for any $\vareps>0$. Note that $e_0(\bmx)=0$ for any $\bmx\in\Omega$, then using Lemma \ref{Gronwall}, we have $\forall \ 0\les t\les T$,
\begin{equation}\label{eq: ineq 2 integrate w.r.t t}
    \|e_t(\cdot)\|^2_{L^2(\Omega;\nu_t)}\les \vareps \int_0^t e^{C_1(\vareps)(t-s)} \|r_s(\cdot)\|^2_{L^2(\Omega;\nu_s)}\rmd s
    :=\vareps L_{\textrm{PINN}}(t;C_1(\vareps)).
\end{equation}
Hence, we have proved \eqref{eq: thm PINN analysis e 1}. 
In addition, if (\ref{eq: thm PINN analysis poincare ineq}) holds, plugging (\ref{eq: thm PINN analysis poincare ineq}) into (\ref{eq: PINN analysis proof ineq 1}), we have
\begin{equation}\label{eq: PINN analysis proof ineq 3}
\pt \|e_t(\cdot)\|_{L^2(\Omega;\nu_t)}^2 \les 
\vareps \|r_t(\cdot)\|^2_{L^2(\Omega;\nu_t)} + C_4(\vareps)\|e_t(\cdot)\|_{L^2(\Omega;\nu_t)}^2.
\end{equation}
Similarly, using Lemma \ref{Gronwall}, we obtain \eqref{eq: thm PINN analysis ineq 3}.
Then we prove \eqref{eq: thm PINN analysis nabla e 1} and \eqref{eq: thm PINN analysis ineq 4}.
From \eqref{eq: PINN analysis proof ineq 1}, we have
\begin{equation}\label{eq: PINN analysis proof second part ineq 1}
m_1\|\nx e_t(\cdot)\|^2_{L^2(\Omega;\nu_t)}\les \vareps\|r_t(\cdot)\|^2_{L^2(\Omega;\nu_t)}+C_1(\vareps)\|e_t(\cdot)\|^2_{L^2(\Omega;\nu_t)}-\pt \|e_t(\cdot)\|^2_{L^2(\Omega;\nu_t)}.
\end{equation}
By Assumption \ref{PINN assumption: continuous assumption}, we bound $\pt \|e_t(\cdot)\|^2_{L^2(\Omega;\nu_t)}$ as follows
\begin{equation}\label{eq: PINN analysis proof second part pt}
\begin{aligned}
& \Big\vert\pt \|e_t(\cdot)\|^2_{L^2(\Omega;\nu_t)}\Big\vert
=  \Big\vert\pt\left(\int_{\Omega}e_t^2(\bmx)\nu_t(\bmx)\rmd \bmx\right)\Big\vert\\
= & \Big\vert\int_{\Omega}e_t^2(\bmx)\pt \nu_t(\bmx)+ 2\nu_t(\bmx)e_t(\bmx)\pt e_t(\bmx)\rmd \bmx\Big\vert\\
\les & \int_{\Omega}e_t^2(\bmx)\nu_t(\bmx)\vert\pt\log \nu_t(\bmx)\vert\rmd\bmx + 2\Big\vert \int_{\Omega}\nu_t(\bmx)e_t(\bmx)\pt e_t(\bmx)\rmd\bmx\Big\vert\\
\les & \ \Bv_0\|e_t(\cdot)\|^2_{L^2(\Omega;\nu_t)} + 2\left(\int_{\Omega}\nu_t(\bmx)e_t^2(\bmx)\rmd\bmx\right)^{1/2}\left(\int_{\Omega}\nu_t(\bmx)\big\vert\pt e_t(\bmx)\big\vert^2\rmd\bmx\right)^{1/2}\\
\les & \ \Bv_0\|e_t(\cdot)\|^2_{L^2(\Omega;\nu_t)} + 2\sqrt{2}\left(\hB_0^2+B_0^{*2}\right)^{1/2}\|e_t(\cdot)\|_{L^2(\Omega;\nu_t)},
\end{aligned}
\end{equation}
which follows from applying $\vert\pt e_t(\bmx)\vert^2=\vert \pt  u_\theta(\bmx, t)-\pt u^*_t(\bmx)\vert^2\les 2\hB_0^2+2B_0^{*2}$. 
Then plugging (\ref{eq: PINN analysis proof second part pt}) into (\ref{eq: PINN analysis proof second part ineq 1}), we have
\begin{equation}
\label{eq: PINN analysis proof second part ineq 2}
m_1\|\nx e_t(\cdot)\|^2_{L^2(\Omega;\nu_t)}\les \vareps\|r_t(\cdot)\|^2_{L^2(\Omega;\nu_t)}+(C_1(\vareps)+\Bv_0)\|e_t(\cdot)\|^2_{L^2(\Omega;\nu_t)} + C_2\|e_t(\cdot)\|_{L^2(\Omega;\nu_t)}.
\end{equation}
Plugging \eqref{eq: thm PINN analysis e 1} and \eqref{eq: thm PINN analysis ineq 3} into (\ref{eq: PINN analysis proof second part ineq 2}) gives \eqref{eq: thm PINN analysis nabla e 1} and \eqref{eq: thm PINN analysis ineq 4} respectively.
\end{proof}

\subsection{Proof of Theorem \ref{thm: Diffusion PINN Sampler analysis}}
\label{appendix: Sampler analysis proof}
Given the $L^2$ error of the score approximation, \citet{chen2023improved} provides an upper bound of  KL divergence between the data distribution $\pi$ and the distribution $\widehat{\pi}_T$ of approximate samples drawn from the sampling dynamics (\ref{eq: special reverse process with discretization}). 
We first summarize the results from \citet{chen2023improved} in Proposition \ref{prop: analysis of generative modeling}.
Then we prove Theorem \ref{thm: Diffusion PINN Sampler analysis} based on Proposition \ref{prop: analysis of generative modeling}.

\begin{proposition}[Theorem 2.5 in \citet{chen2023improved}]\label{prop: analysis of generative modeling}
Suppose that $T\ges 1$, $K\ges 2$, and the $L^2$ error of the score approximation is bounded by
\begin{equation}\label{eq: score approximation error}
\sum_{k=1}^Nh_k\mathbb{E}_{\bmx_{t_k}\sim\pi_{t_k}}\left\|\nx\log\pi_{t_k}(\bmx_{t_k})-\bms_{t_k}(\bmx_{t_k})\right\|^2\les T\vareps_0^2.
\end{equation}
Then there is a universal constant $\alpha \ges 2$ such that the following holds. 
Under Assumption \ref{Sampler assumption: target distr}, by using the exponentially decreasing (then constant) step size
$h_k=h\min\{\max\{t_k,1/(4K)\}, 1\}$, $0<h\les 1/(\alpha d)$, the sampling dynamic (\ref{eq: special reverse process with discretization}) results in a distribution $\widehat{\pi}_T$ such that 
\begin{equation}\label{eq: KL bounded in analysis generative modeling}
\operatorname{KL}(\pi\| \widehat{\pi}_T)\lesssim 
(d+M_2)\cdot e^{-T} + 
T\vareps_0^2 +
d^2h(\log K + T),
\end{equation}
where the number of sampling steps satisfies that $N\lesssim\frac{1}{h}(\log K + T)$. Choosing $T=\log \left(\frac{M_2+d}{\vareps_0^2}\right)$ and $h=\Theta\left(\frac{\vareps_0^2}{d^2(\log K+T)}\right)$, we have $N=\mathcal{O}\left(\frac{d^2(\log K+T)^2}{\vareps_0^2}\right)$ and make the KL divergence $\widetilde{\mathcal{O}}\left(\vareps_0^2\right)$.
\end{proposition}

\begin{proof}[Proof of Theorem \ref{thm: Diffusion PINN Sampler analysis}]
As $\bms_t(\bmx)=\nx u_\theta(\bmx, t)\cdot\bbmone\{\bmx\in\Omega\}$, we have
\begin{equation}\label{eq: score approximation bounded by PINNs}
\begin{aligned}
& \sum_{k=1}^Nh_k\mathbb{E}_{\bmx_{t_k}\sim\pi_{t_k}}\left\|\nx\log\pi_{t_k}(\bmx_{t_k}) - \bms_{t_k}(\bmx_{t_k})\right\|^2\\
= & \sum_{k=1}^Nh_k\int_{\Omega^c}\pi_{t_k}(\bmx)\|\nx\log\pi_{t_k}(\bmx)\|^2\rmd\bmx + 
\sum_{k=1}^N h_k\int_{\Omega}\pi_{t_k}(\bmx)\|\nx e_{t_k}(\bmx)\|^2\rmd\bmx\\
\les & \sum_{k=1}^N h_k\delta + \sum_{k=1}^Nh_kR_{t_k}\|\nx e_{t_k}(\cdot)\|^2_{L^2(\Omega;\nu_{t_k})}\quad \text{(using Assumption \ref{Sampler assumption: bounded region} and \ref{Sampler assumption: base target ratio})}\\
\les & \ T\delta + \delta_1 + C_5(\vareps)\delta_2 + C_2\sqrt{\sum_{k=1}^Nh_kR_{t_k}\delta_2},
\end{aligned}
\end{equation}
where the last inequality follows from Theorem \ref{thm: PINN analysis} ($m_1=M_1=1$) and
\begin{equation}\label{eq: Cauchy ineq used for Sampler analysis proof}
\begin{aligned}
\sum_{k=1}^Nh_kR_{t_k}\sqrt{\vareps L_{\textrm{PINN}}(t_k; C_1(\vareps))}\les & \left(\vareps\sum_{k=1}^Nh_kR_{t_k}L_{\textrm{PINN}}(t_k; C_1(\vareps))\right)^{1/2}\left(\sum_{k=1}^Nh_kR_{t_k}\right)^{1/2}\\
\les & \ \sqrt{\sum_{k=1}^Nh_kR_{t_k}\delta_2}.
\end{aligned}
\end{equation}
Then combining \eqref{eq: score approximation bounded by PINNs} and Proposition \ref{prop: analysis of generative modeling} together gives the results in Theorem \ref{thm: Diffusion PINN Sampler analysis}.
\end{proof}

\subsection{Proof of Theorem \ref{thm: extended PINN analysis for uniform sampling}}
\label{appendix: proof of thm: extended PINN for uniformed sampling}
Note that $u_t^*(\bmx)$ satisfies
\begin{equation}\label{eq: Unif pde for u^*}
\pt u^*_t(\bmx) 
+ \nx u^*_t(\bmx)\cdot\bmf(\bmx, t) 
+ \diver \bmf(\bmx, t) 
- \frac{1}{2}g^2(t)\Lap u_t^*(\bmx)
- \frac{1}{2}g^2(t)\left\|\nx u_t^*(\bmx)\right\|^2
= 0.
\end{equation}
And $ u_\theta(\bmx, t)$ satisfies
\begin{equation}\label{eq: Unif pde for hu}
\pt  u_\theta(\bmx, t) 
+ \nx  u_\theta(\bmx, t)\cdot\bmf(\bmx, t) 
+ \diver \bmf(\bmx, t) 
- \frac{1}{2}g^2(t)\Lap  u_\theta(\bmx, t)
- \frac{1}{2}g^2(t)\left\|\nx  u_\theta(\bmx, t)\right\|^2
= r_t(\bmx).
\end{equation}
Subtracting \eqref{eq: Unif pde for u^*} for $u^*$ from \eqref{eq: Unif pde for hu} for $u_\theta$, we have
\begin{equation}\label{eq: Unif PINN analysis proof e_t's PDE}
\partial_t e_t(\bmx)
+\nabla_{\bmx} e_t(\bmx)\cdot \bmf(\bmx, t)
-\frac{1}{2}g^2(t)\left(\left\|\nabla_{\bmx}   u_\theta(\bmx, t)\right\|^2 - \left\|\nabla_{\bmx} u^*_t(\bmx)\right\|^2\right)
-\frac{1}{2}g^2(t)\Delta e_t(\bmx)=r_t(\bmx).
\end{equation}
Note that $\frac{1}{2}\partial_t e_t^2(\bmx)=e_t(\bmx)\partial_t e_t(\bmx)$ and $\frac{1}{2}\nabla_{\bmx} e_t^2(\bmx)=e_t(\bmx)\nabla_{\bmx} e_t(\bmx)$, then we obtain
\begin{equation}\label{eq: Unif PINN analysis proof partial_t e_t}
\begin{aligned}
\frac{1}{2}\partial_t e_t^2(\bmx) 
= & \ \frac{1}{2}g^2(t)e_t(\bmx)\left(\left\|\nabla_{\bmx}   u_\theta(\bmx, t)\right\|^2 - \left\|\nabla_{\bmx} u^*_t(\bmx)\right\|^2\right) 
+ \frac{1}{2}g^2(t)e_t(\bmx)\Delta e_t(\bmx) \\
& \quad + e_t(\bmx)r_t(\bmx) 
- e_t(\bmx)\nabla_{\bmx} e_t(\bmx)\cdot \bmf(\bmx, t)\\
= & \ \frac{1}{2}g^2(t)e_t(\bmx)\nabla_{\bmx} e_t(\bmx)\cdot\left(\nabla_{\bmx}   u_\theta(\bmx, t) + \nabla_{\bmx} u^*_t(\bmx)\right) 
+ \frac{1}{2}g^2(t)e_t(\bmx)\Delta e_t(\bmx) \\
&\quad + e_t(\bmx)r_t(\bmx) 
- e_t(\bmx)\nabla_{\bmx} e_t(\bmx)\cdot \bmf(\bmx, t)\\
= & \ \frac{1}{4}g^2(t)\nabla_{\bmx} e_t^2(\bmx)\cdot\left(\nabla_{\bmx}   u_\theta(\bmx, t) + \nabla_{\bmx} u^*_t(\bmx)\right) 
+ \frac{1}{2}g^2(t)e_t(\bmx)\Delta e_t(\bmx)\\
&\quad + e_t(\bmx)r_t(\bmx)
- \frac{1}{2}\nabla_{\bmx} e_t^2(\bmx)\cdot f(\bmx, t).
\end{aligned}
\end{equation}
We integrate (\ref{eq: Unif PINN analysis proof partial_t e_t}) to get an equation for $\|e_t(\cdot)\|^2_{L^2(\Omega)}$ given by 
\begin{equation}\label{eq: Unif PINN analysis proof ev integrate}
\begin{aligned}
\pt\left\|e_t(\cdot)\right\|^2_{L^2(\Omega)}
= &\ \frac{1}{2}g^2(t)\int_{\Omega}\nabla_{\bmx} e_t^2(\bmx)\cdot\left(\nabla_{\bmx}   u_\theta(\bmx, t) + \nabla_{\bmx} u^*_t(\bmx)\right) \rd\bmx + \ g^2(t)\int_{\Omega}e_t(\bmx)\Delta e_t(\bmx) \rd\bmx\\
& \quad + 2\int_{\Omega}e_t(\bmx)r_t(\bmx)\rd\bmx
- \int_{\Omega}\nabla_{\bmx} e_t^2(\bmx)\cdot\bmf(\bmx, t) \rd\bmx\\
= &\ - \frac{1}{2}g^2(t)\int_{\Omega}e_t^2(\bmx)\cdot\left(\Delta  u_\theta(\bmx, t) + \Delta u^*_t(\bmx)\right) \rd\bmx - g^2(t)\int_{\Omega}\left\|\nabla_{\bmx}e_t(\bmx)\right\|^2 \rd\bmx\\
& \quad + 2\int_{\Omega}e_t(\bmx)r_t(\bmx)\rd\bmx
+ \int_{\Omega}e_t^2(\bmx)\cdot\left[\nabla\cdot\bmf(\bmx, t)\right] \rd\bmx\\
\les &\ - \frac{m_1}{2}\left(B_2^*+\widehat{B}_2\right)\left\|e_t(\cdot)\right\|^2_{L^2(\Omega)} - m_1\left\|\nx e_t(\cdot)\right\|_{L^2(\Omega)}^2 + \eps\left\|r_t(\cdot)\right\|^2_{L^2(\Omega)}\\
&\ \quad + \frac{1}{\eps}\left\|e_t(\cdot)\right\|^2_{L^2(\Omega)} + \mathrm{m}_2\left\|e_t(\cdot)\right\|^2_{L^2(\Omega)}\\
= &\ C_1^{\textrm{U}}(\eps)\left\|e_t(\cdot)\right\|_{L^2(\Omega)}^2 + \eps\left\|r_t(\cdot)\right\|_{L^2(\Omega)}^2 - m_1\left\|\nx e_t(\cdot)\right\|^2_{L^2(\Omega)}\\
\les &\ C_1^{\textrm{U}}(\eps)\left\|e_t(\cdot)\right\|_{L^2(\Omega)}^2 + \eps\left\|r_t(\cdot)\right\|_{L^2(\Omega)}^2.
\end{aligned}
\end{equation}
Note that $e_0(\bmx)=0$ for any $\bmx\in\Omega$, then using the Gr\"{o}nwall inequality, we have for any $t\in[0, T]$,
\begin{equation}\label{eq: Unif ineq 2 integrate w.r.t t}
    \|e_t(\cdot)\|^2_{L^2(\Omega)}\les \eps \int_0^t e^{C_1^{\textrm{U}}(\eps)(t-s)} \left\|r_s(\cdot)\right\|^2_{L^2(\Omega)}\rd s
    :=\eps L_{\textrm{PINN}}^{\textrm{Unif}}(t;C_1^{\textrm{U}}(\eps)).
\end{equation}
Note that from \eqref{eq: Unif PINN analysis proof ev integrate},
\begin{equation}\label{eq: Unif PINN analysis proof second part ineq 1}
m_1\|\nx e_t(\cdot)\|^2_{L^2(\Omega)}\les \eps\|r_t(\cdot)\|^2_{L^2(\Omega)}+C_1^{\textrm{U}}(\eps)\|e_t(\cdot)\|^2_{L^2(\Omega)}-\pt \|e_t(\cdot)\|^2_{L^2(\Omega)}.
\end{equation}
We can bound $\pt \|e_t(\cdot)\|^2_{L^2(\Omega)}$ as follows
\begin{equation}\label{eq: Unif PINN analysis proof second part pt}
\begin{aligned}
& \left\vert\pt \|e_t(\cdot)\|^2_{L^2(\Omega)}\right\vert
=  \left\vert\pt\left(\int_{\Omega}e_t^2(\bmx)\rd \bmx\right)\right\vert
=  2\left\vert\int_{\Omega}e_t(\bmx)\pt e_t(\bmx)\rd \bmx\right\vert\\
\les &\ 2\left(\int_{\Omega}e_t^2(\bmx)\rd\bmx\right)^{1/2}\left(\int_{\Omega}\left\vert\pt e_t(\bmx)\right\vert^2\rd\bmx\right)^{1/2}
\les 2\sqrt{2}\left(\hB_0^2+B_0^{*2}\right)^{1/2}\|e_t(\cdot)\|_{L^2(\Omega)},
\end{aligned}
\end{equation}
which follows from applying $\vert\pt e_t(\bmx)\vert^2=\vert \pt  u_\theta(\bmx, t)-\pt u^*_t(\bmx)\vert^2\les 2\hB_0^2+2B_0^{*2}$. 
Then plugging (\ref{eq: Unif PINN analysis proof second part pt}) into (\ref{eq: Unif PINN analysis proof second part ineq 1}), we have
\begin{equation}
\label{eq: Unif PINN analysis proof second part ineq 2}
m_1\|\nx e_t(\cdot)\|^2_{L^2(\Omega)}\les \eps\|r_t(\cdot)\|^2_{L^2(\Omega)}+C_1^{\textrm{U}}(\eps)\|e_t(\cdot)\|^2_{L^2(\Omega)} + C_2  \|e_t(\cdot)\|_{L^2(\Omega)}.
\end{equation}
Plugging \eqref{eq: Unif PINN analysis proof second part ineq 1} into \eqref{eq: Unif PINN analysis proof second part ineq 2}, we complete the proof.

\subsection{Proof of Theorem \ref{thm: Unif KL bound}}
\label{appendix: proof of thm: Unif KL bound}
As $\bms_t(\bmx)=\nx u_\theta(\bmx, t)\cdot\1\{\bmx\in\Omega\}$, we have
\begin{equation}\label{eq: Unif score approximation bounded by PINNs}
\begin{aligned}
& \sum_{k=1}^Nh_k\mathbb{E}_{\bmx_{t_k}\sim\pi_{t_k}}\left\|\nx\log\pi_{t_k}(\bmx_{t_k}) - \bms_{t_k}(\bmx_{t_k})\right\|^2\\
= & \sum_{k=1}^Nh_k\int_{\Omega^c}\pi_{t_k}(\bmx)\|\nx\log\pi_{t_k}(\bmx)\|^2\rd\bmx + 
\sum_{k=1}^N h_k\int_{\Omega}\pi_{t_k}(\bmx)\|\nx e_{t_k}(\bmx)\|^2\rd\bmx\\
\les & \sum_{k=1}^N h_k\delta + \sum_{k=1}^Nh_k\max_{\bmx\in\Omega}\left\{\pi_{t_k}(\bmx)\right\}\cdot\|\nx e_{t_k}(\cdot)\|^2_{L^2(\Omega)}\\
\les & \ T\delta + \delta_1\cdot\textrm{Vol}(\Omega) + C_1^{\textrm{U}}(\eps)\delta_2\cdot\textrm{Vol}(\Omega) + C_2 \sqrt{\sum_{k=1}^Nh_k\max_{\bmx\in\Omega}\left\{\pi_{t_k}(\bmx)\right\}\cdot\delta_2\cdot \textrm{Vol}(\Omega)},
\end{aligned}
\end{equation}
where the last inequality follows from the result in Theorem \ref{thm: extended PINN analysis for uniform sampling} and
\begin{equation}\label{eq: Unif Cauchy ineq used for Sampler analysis proof}
\begin{aligned}
&\ \sum_{k=1}^Nh_k\max_{\bmx\in\Omega}\left\{\pi_{t_k}(\bmx)\right\}\cdot \sqrt{\eps L_{\textrm{PINN}}^{\textrm{Unif}}(t_k; C_1^{\textrm{U}}(\eps))}\\
\les &\ \left(\eps\sum_{k=1}^Nh_k\max_{\bmx\in\Omega}\left\{\pi_{t_k}(\bmx)\right\}\cdot L_{\textrm{PINN}}^{\textrm{Unif}}(t_k; C_1^{\textrm{U}}(\eps))\right)^{1/2}\left(\sum_{k=1}^Nh_k\max_{\bmx\in\Omega}\left\{\pi_{t_k}(\bmx)\right\}\right)^{1/2}\\
\les &\ \sqrt{\sum_{k=1}^Nh_k\max_{\bmx\in\Omega}\left\{\pi_{t_k}(\bmx)\right\}\cdot\delta_2\cdot \textrm{Vol}(\Omega)}.
\end{aligned}
\end{equation}
Combining \eqref{eq: Unif score approximation bounded by PINNs} and Proposition \ref{prop: analysis of generative modeling}, we complete our proof.

\section{Limitations}
\label{appendix: limitations}
As we use LMC for collocation generation in DPS, there is a risk of missing modes if short LMC runs do not adequately cover the high-density domain.
In such cases, running LMC for an annealed path of target distributions or adopting the adversarial training method in \cite{wang20222} for collocation points maybe helpful.
Also, solving high dimensional PDEs via PINN can be challenging, and we may use techniques such as stochastic dimension gradient descent or the Hutchinson trick to scale DPS to high dimensional problems \citep{hu2024tackling,hu2024hutchinson}.

\section{Additional Experimental Details and Results}
\label{appendix: experimental details and results}

\subsection{Baselines}
\label{appendix: details on baselines}
We benchmark DPS performance against a wide range of strong baseline methods. 
For MCMC methods, we consider the Langevin Monte Carlo (LMC). 
For LMC, we run 100,000 iterations with step sizes 0.02, 0.002, 0.0002. 
Then we choose the samples with the best performance.
As for sampling methods using reverse diffusion, we include RDMC \citep{huang2023reverse}, and SLIPS \citep{grenioux2024stochastic}.
We use the implementation of SLIPS and RDMC from \cite{grenioux2024stochastic} and choose Geom($1,1$) as the SL scheme for SLIPS.
For each algorithm, we search its hyper-parameters within a predetermined grid, similar to \citet{grenioux2024stochastic}. 
We also compare with VI-based PIS \citep{zhang2021path} and DIS \citep{berner2022optimal}.
We use the implementation of PIS and DIS from \cite{berner2022optimal}.

\subsection{Targets}
\label{appendix: details on targets}
\textbf{9-Gaussians} is a 2-dimensional Mixture of Gaussians where there are 9 modes designed to be well-separated from each other. 
The modes share the same variance of 0.3 and the means are located in the grid of $\{-5, 0, 5\}\times\{-5, 0, 5\}$. 
We set challenging mixing proportions between different modes as shown in Table \ref{table:relative weights in 9-Gaussians}.

\textbf{Rings} is the inverse polar reparameterization of a $2$-dimensional distribution $p_z$ which has itself a decomposition
into two univariate marginals $p_r$ and $p_{\theta}$: $p_r$ is a mixture of $4$ Gaussian distributions $\mathcal{N}(i, 0.2^2)$ with $i=2, 4, 6, 8$ describing the radial position and $p_{\theta}$ is a uniform distribution over $[0, 2\pi)$, which describes the angular position of the samples. 
We also set challenging mixing proportions between different modes of $p_r$ as shown in Table \ref{table:relative weights in rings}.

\begin{table}[t]
    \caption{Mixing proportions between 9 modes in 9-Gaussians.}
    \label{table:relative weights in 9-Gaussians}
    \centering
    \scalebox{0.9}{
    \begin{tabular}{c  c  c  c  c  c  c  c  c  c}\toprule
          Modes & 
          $(-5,-5)^{\prime}$ &
          $(-5,0)^{\prime}$ &
          $(-5,5)^{\prime}$ &
          $(0,-5)^{\prime}$ &
          $(0,0)^{\prime}$ &
          $(0,5)^{\prime}$ &
          $(5,-5)^{\prime}$ &
          $(5,0)^{\prime}$ &
          $(5,5)^{\prime}$ 
         \\
         Weight & 
          $0.2$ &
          $0.04$ &
          $0.2$ &
          $0.04$ &
          $0.04$ &
          $0.04$ &
          $0.2$ &
          $0.04$ &
          $0.2$ 
          \\
        \bottomrule
    \end{tabular}
    }
\end{table}

\begin{table}[t]
    \caption{Mixing proportions between 4 modes in rings.}
    \label{table:relative weights in rings}
    \centering
    \scalebox{1.0}{
    \begin{tabular}{c  c  c  c  c }\toprule
          Modes & 
          $r=2$ &
          $r=4$ &
          $r=6$ &
          $r=8$ 
         \\
         Weight & 
          $0.05$ &
          $0.45$ &
          $0.05$ &
          $0.45$ 
          \\
        \bottomrule
    \end{tabular}
    }
\end{table}

\textbf{Funnel} is a classical sampling benchmark problem from \citet{neal2003slice, hoffman2014no}.
This 10-dimensional density is defined by 
\[\mu(\bmx):=\mathcal{N}(x_0;0,9)\mathcal{N}(\bmx_{1:9};\bm{0},\exp(x_0)\bm{I}_9).\]

\textbf{Double-well} is a high-dimensional distribution which share the unnormalized density:
\[
\mu(\bmx):=\exp\left(\sum_{i=0}^{w-1}-x_i^4+6x_i^2+0.5x_i-\sum_{i=w}^{d-1}0.5x_i^2\right).
\] 
We choose $w=3$ and $d=30$, leading to a $30$-dimensional distribution contained $8$ modes with challenging mixing proportions between different modes.

\subsection{Diffusion-PINN Sampler}
\label{appendix: DPS}

\paragraph{Model.} The model architecture of $\mathrm{NN}_{\theta}(\bmx, t):\bR^d\times [0, T]\to\bR$ in $ u_\theta(\bmx, t)$ is 
\[
\mathrm{NN}_{\theta}(\bmx, t) = \mathrm{MLP}^{\textrm{dec}}\left(\mathrm{MLP}^{\text{embx}}(\bmx) + \mathrm{MLP}^{\text{embt}}(\mathrm{emb}(t))\right),
\]
where $\mathrm{MLP}^{\textrm{dec}}$ represents a decoder implemented as MLPs with layer widths $[128, 128, 128, 1]$. The component $\mathrm{MLP}^{\text{embx}}$ serves as a data embedding block and is implemented as MLPs with layer widths $[2, 128]$.
$\mathrm{MLP}^{\text{embt}}$ functions as a time embedding block, implemented as MLPs with layer widths $[256, 128, 128]$. The input to $\mathrm{MLP}^{\text{embt}}$ is derived from the sinusoidal positional embedding~\citep{vaswani2017} of $t$. 
All these three MLPs utilize the GELU activation function.

\paragraph{Training.}
In our implementation, we choose $\bmf(\bmx,t)=-\frac{\bmx}{2(1-t)}$ and $g(t)=\sqrt{\frac{1}{1-t}}$, leading to the following forward process
\begin{equation}\label{forward process in implementation}
\rmd\bmx_t = -\frac{\bmx_t}{2(1-t)}\rmd t + \sqrt{\frac{1}{1-t}}\rmd\bm{B}_t, \quad \bmx_0\sim\pi, \quad T_{\textrm{min}}\les t\les T_{\textrm{max}}.
\end{equation}
This admits the explicit conditional distribution $\pi_{t\mid 0}(\bmx_t\vert \bmx_0)=\mathcal{N}(\bmx_t;\sqrt{1-t}\cdot\bmx_0, t\cdot\bm{I}_d)$.
We choose $T_{\textrm{min}}=0.001$ and $T_{\textrm{max}}=0.999$ in practice. 
The corresponding log-density FPE becomes 
\begin{equation}\label{eq: log-density FPE implemented}
    \pt u_t(\bmx)= \frac{1}{2(1-t)}\left[
    \Lap u_t(\bmx) + \left\|\nx u_t(\bmx)\right\|^2 + \bmx\cdot\nx u_t(\bmx) + d
    \right]:=\frac{1}{2(1-t)}\mathcal{L}_{\textrm{L-FPE}}^{\textrm{prac}}u_t(\bmx).
\end{equation}
We choose $\beta(t)=2(1-t)$ to make training more stable, leading to the following training objective
\begin{equation}\label{eq: training objective implemented}
\begin{aligned}
    L_{\textrm{train}}^{\textrm{prac}}(u_\theta):= & \bE_{t\sim\mathcal{U}[T_{\textrm{min}},T_{\textrm{max}}]}\bE_{\bmx_t\sim \nu_t}\left[
    \left\|2(1-t)\cdot \pt   u_\theta(\bmx_t, t) - \mathcal{L}_{\textrm{L-FPE}}^{\textrm{prac}}  u_\theta(\bmx_t, t)\right\|^2
    \right] \\
    &\quad + \lambda \cdot \bE_{\bmz\sim\mathcal{N}(\bm{0}, \bm{I}_d)}\left[
    \left\|\nabla_{\bmz} u_\theta(\bmz, T_{\textrm{max}}) + \bmz\right\|^2
    \right],
\end{aligned}
\end{equation}
where $\lambda$ is the regularization coefficient. 
It is enough for us to use PINN residual loss without regularization except for Funnel where the regularization is quite useful and we use $\lambda=1$.
To generate collocation points for PINN, we run a short chain of LMC with a large step size.
The hyper-parameters used in LMC for different targets are reported in Table \ref{table: hyper-parameters for LMC for collocation points}.
We generate fresh collocation points per iteration except for Funnel where we resample new collocation points per $10,000$ iterations.

\begin{table}[t]
    \caption{Hyper-parameters used in LMC for generating collocation points.}
    \label{table: hyper-parameters for LMC for collocation points}
    \centering
    \scalebox{1.0}{
    \begin{tabular}{c  c  c  c  c }\toprule
           &
           9-Gaussians &
           Rings &
           Funnel &
           Double-well
       \\ \midrule
          step size & 
          $1.0$ &
          $0.15$ &
          $0.02$ &
          $0.02$ 
         \\
         iterations & 
          $60$ &
          $100$ &
          $10,000$ &
          $100$ 
          \\
          batch size &
          $128$ &
          $200$ &
          $200$ &
          $700$ 
          \\
          refresh samples per iteration &
          \Checkmark &
          \Checkmark &
          \XSolidBrush &
          \Checkmark
          \\
        \bottomrule
    \end{tabular}
    }
\end{table}

We train all models with Adam optimizer \citep{kingma2014adam}. 
The hyper-parameters used in training are summarized in Table \ref{table: hyper-parameters for training}. 
We use a linear decay schedule for the learning rate in all experiments.

\begin{table}[t]
    \caption{Hyper-parameters for training PINN.}
    \label{table: hyper-parameters for training}
    \centering
    \scalebox{1.0}{
    \begin{tabular}{c  c  c  c  c }\toprule
           &
           9-Gaussians &
           Rings &
           Funnel &
           Double-well
       \\ \midrule
          learning rate & 
          $0.0005$ &
          $0.0005$ &
          $0.0001$ &
          $0.0005$ 
         \\
         max norm of gradient clipping &
         $1.0$ &
         $1.0$ &
         $1000.0$ &
         $1.0$
         \\
         regularization coefficient $\lambda$ &
         $0$ &
         $0$ &
         $1$ &
         $0$
         \\
          total training iterations &
          $400$k &
          $1,000$k &
          $800$k &
          $1,500$k
          \\
        \bottomrule
    \end{tabular}
    }
\end{table}

\paragraph{Sampling.}The corresponding reverse process is given by 
\begin{equation}\label{eq: reverse process in practice}
    \rmd\bmx_t = \left(\frac{\bmx_t}{2t}+\frac{\nx\log \pi_{1-t}(\bmx_t)}{t}\right)\rmd t + \sqrt{\frac{1}{t}}\rmd\bm{B}_t,\quad \bmx_0\sim \pi_{T_{\textrm{max}}},\quad T_{\textrm{min}}\les t\les T_{\textrm{max}}.
\end{equation}
To simulate \eqref{eq: reverse process in practice}, we approximate $\pi_{T_{\textrm{max}}}\approx \mathcal{N}(\bm{0},\bm{I}_d)$ and use the exponential integrator scheme with the score approximation $\bms_t(\bmx)\approx\nx\log \pi_t(\bmx)$.
In practice, we use $\bms_t(\bmx):=\nx u_\theta(\bmx, t)\cdot \bbmone\{\bmx\in\Omega\}$ where $ u_\theta(\bmx, t)$ is the approximated log-density provided by the PINN approach (trained by Algorithm \ref{algorithm: PINN}) and $\Omega$ is a chosen bounded region that covers the high density domain of $\pi_t$ for any $t\in [T_{\textrm{min}},T_{\textrm{max}}]$.
We use $\Omega:=\{\bmx\in\bR^d:\|\bmx\|\les R\}$ in all experiments, the choice of $R$ is reported in Table \ref{table: bounded region choice}.
Our sampling process is summarized in Algorithm \ref{algorithm: sampling}.
We provide more sampling performances of different methods for different targets in Figure \ref{fig:samples visualization full version} and sample trajectories from DPS in Figure \ref{fig:samples visualization trajectories}.

\begin{table}[H]
    \caption{The diameter of the truncated region for different targets.}
    \label{table: bounded region choice}
    \centering
    \scalebox{1.0}{
    \begin{tabular}{c  c  c  c  c }\toprule
           &
           9-Gaussians &
           Rings &
           Funnel &
           Double-well
       \\ \midrule
          $R$ & 
          $20$ &
          $20$ &
          $2000$ &
          $30$ 
         \\
        \bottomrule
    \end{tabular}
    }
\end{table}

\begin{algorithm}[t!]
    \caption{: Sampling from reverse process}
    \label{algorithm: sampling}
    \begin{algorithmic}[1]
        \REQUIRE Starting time $T_{\textrm{min}}$, Terminal time $T_{\textrm{max}}$, Sample size $M$, Discretization steps $N$, Bounded domain $\Omega$, Approximated log-density $ u_\theta(\bmx, t)$ provided by PINN.
        \STATE Compute the step size $h:=(T_{\textrm{max}} - T_{\textrm{min}}) / N$.
        \STATE Obtain the approximated score function $\bms_t(\bmx):=\nx u_\theta(\bmx, t)\cdot\bbmone\{\bmx\in\Omega\}$.
        \STATE Sample i.i.d. $\bmx_i^0\sim\mathcal{N}(\bm{0},\bm{I}_d),\ \forall 1\les i\les M$
        \FOR{$n=1,\cdots, N$}
        \STATE Sample i.i.d. $\bmz_i\sim \mathcal{N}(\bm{0},\bm{I}_d),\ \forall 1\les i\les M$.
        \STATE Compute $t_{n-1}:=T_{\textrm{min}}+(n-1)h$.
        \STATE Update by simulating the reverse process: $\forall 1\les i\les M$
        \[
        \bmx_i^{n}\leftarrow 
        \sqrt{1+\frac{h}{t_{n-1}}}\bmx_i^{n-1} + 2\left(\sqrt{1+\frac{h}{t_{n-1}}}-1\right)\bms_{1-t_{n-1}}(\bmx_i^{n-1}) + 
        \sqrt{\frac{h}{t_{n-1}}}\bmz_i,
        \]
        \ENDFOR
        \RETURN Approximated samples $\bmx_1^N,\cdots, \bmx_M^N$.
    \end{algorithmic}
\end{algorithm}

\begin{figure}[H]
    \centering
    \includegraphics[width=\linewidth]{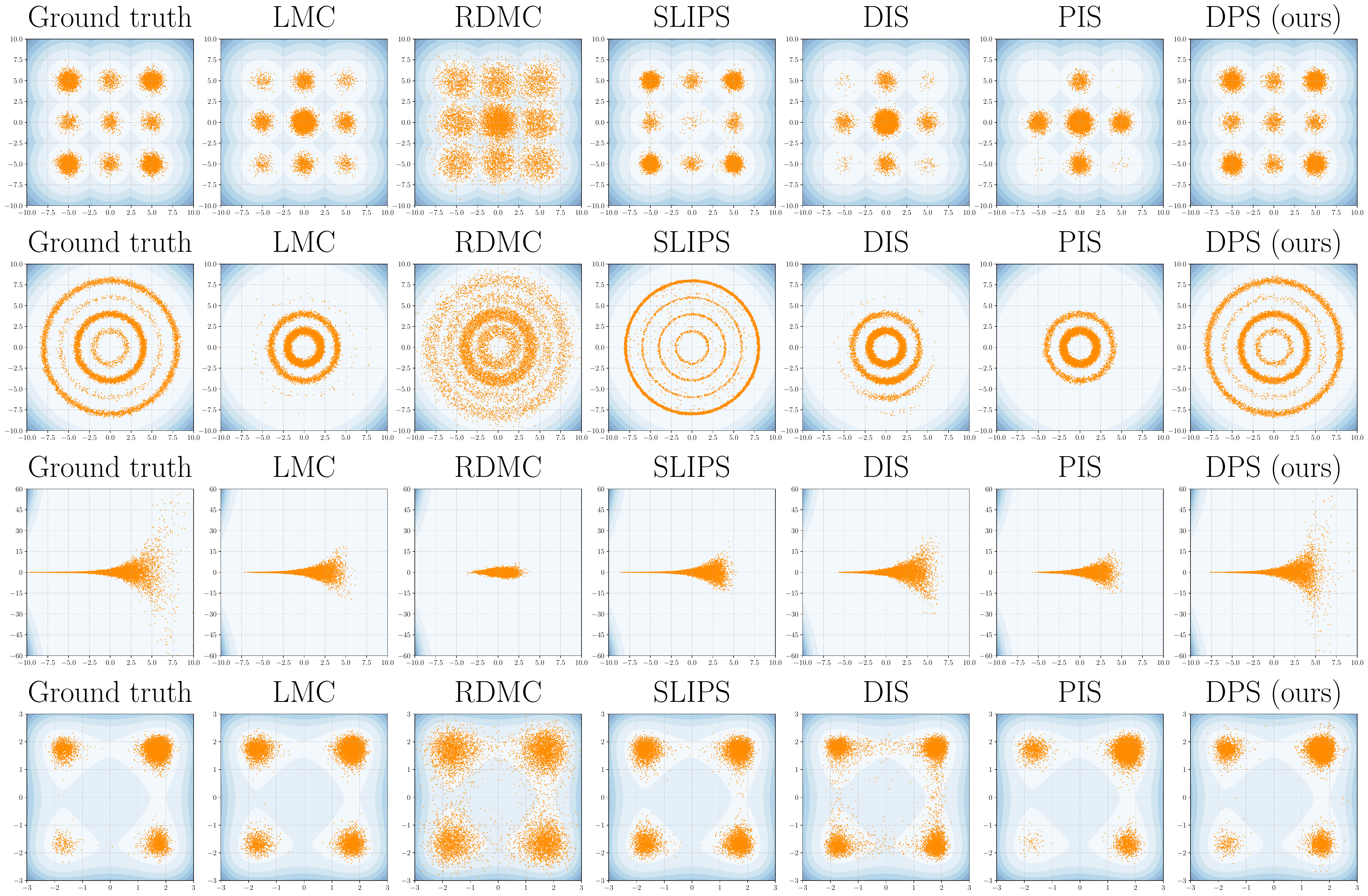}
    \caption{Sampling performance of different methods for $9$-Gaussians ($d=2$), Rings ($d=2$), Funnel ($d=10$) and Double-well ($d=30$).}
    \vspace{-0.5cm}
    \label{fig:samples visualization full version}
\end{figure}

\begin{figure}[H]
    \centering
    \includegraphics[width=\linewidth]{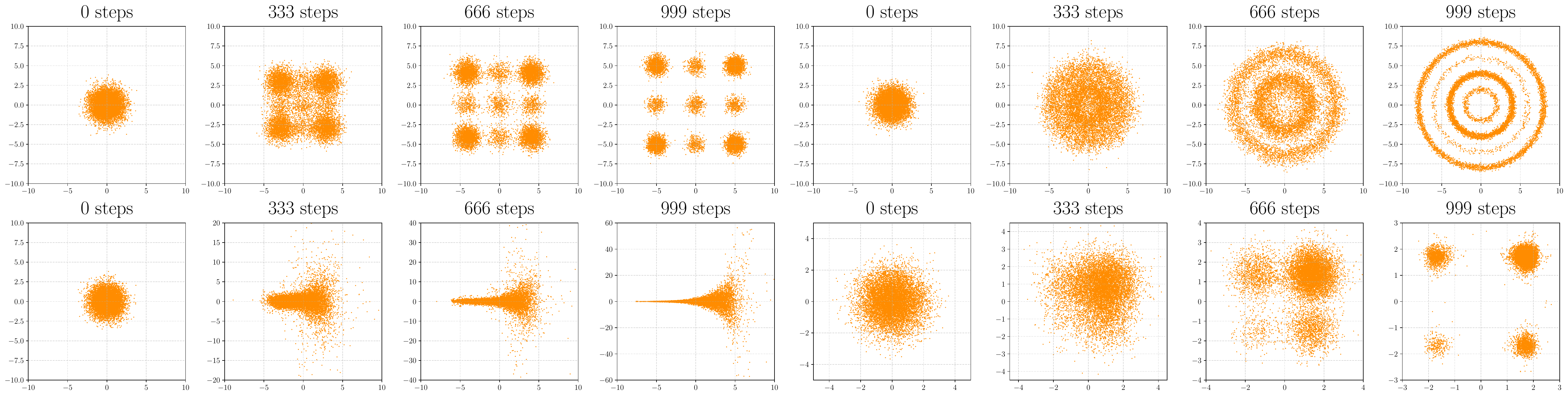}
    \vspace{-0.5cm}
    \caption{Sample trajectories from DPS for $9$-Gaussians ($d=2$), Rings ($d=2$), Funnel ($d=10$) and Double-well ($d=30$).}
    \label{fig:samples visualization trajectories}
\end{figure}



\end{document}